\documentclass[9pt,twocolumn,twoside]{pnas-new}

\newcommand{\figstart}{The formatting and technical details are as described in Section \ref{sec:figure_description}. }

\usepackage{bbm}
\usepackage{array}
\usepackage{animate}
\usepackage{amsmath}
\usepackage{amssymb}
\usepackage{physics}
\usepackage{graphicx}
\usepackage{subcaption}
\usepackage{placeins}
\usepackage{hyperref}
\usepackage[abs]{overpic}
\usepackage{varwidth}
\usepackage{amsthm}
\usepackage{bm}
\usepackage{soul}

\definecolor{darkgreen}{rgb}{0.0, 0.2, 0.13}
\definecolor{darkblue}{rgb}{0.0, 0.0, 0.55}
\definecolor{darkred}{rgb}{0.55, 0.0, 0.0}

\definecolor{tableau_blue}{HTML}{1F77B4}
\definecolor{tableau_orange}{HTML}{FF7F0E}
\definecolor{tableau_green}{HTML}{2CA02C}
\definecolor{tableau_red}{HTML}{D62728}
\definecolor{tableau5}{HTML}{9467BD}
\definecolor{tableau6}{HTML}{8C564B}
\definecolor{tableau7}{HTML}{CFECF9}
\definecolor{tableau8}{HTML}{7F7F7F}
\definecolor{tableau9}{HTML}{BCBD22}
\definecolor{tableau10}{HTML}{17BECF}

\theoremstyle{remark}

\newtheorem{theorem}{Theorem}
\newtheorem{proposition}[theorem]{Proposition}

\newtheorem{definition}{Definition}

\DeclareMathOperator*{\E}{\mathbb{E}}
\DeclareMathOperator*{\Ave}{Ave}
\DeclareMathOperator*{\argmax}{arg\,max}
\DeclareMathOperator*{\argmin}{arg\,min}

\definecolor{cyan}{RGB}{23,190,207}
\definecolor{blue}{RGB}{31,119,180}
\definecolor{orange}{RGB}{255,127,14}
\definecolor{red}{RGB}{214,39,40}

\newcommand{\ETFfig}[2]{{#1}} 

\newcommand{\zr}{\mathbf{0}}
\newcommand{\R}{\mathbb{R}}
\newcommand{\one}{{\mathbbm{1}}}
\newcommand{\M}{\boldsymbol{M}}
\newcommand{\Mc}{\dot{\boldsymbol{M}}}
\newcommand{\Mstar}{{\M}^\star}
\newcommand{\U}{\boldsymbol{U}}
\newcommand{\A}{\boldsymbol{A}}
\renewcommand{\S}{\boldsymbol{S}}
\newcommand{\V}{\boldsymbol{V}}
\newcommand{\St}{\boldsymbol{\Sigma}_T}
\newcommand{\Sb}{\boldsymbol{\Sigma}_B}
\newcommand{\Sw}{\boldsymbol{\Sigma}_W}
\newcommand{\bmu}{\boldsymbol{\mu}}
\newcommand{\btheta}{\boldsymbol{\theta}}
\newcommand{\x}{{\boldsymbol{x}}}
\newcommand{\h}{\boldsymbol{h}}
\newcommand{\e}{\boldsymbol{e}}
\newcommand{\I}{\boldsymbol{I}}

\renewcommand{\r}{{\boldsymbol{r}}}
\newcommand{\y}{\boldsymbol{y}}
\renewcommand{\L}{\mathcal{L}}
\newcommand{\W}{\boldsymbol{W}}
\renewcommand{\H}{\boldsymbol{H}}
\newcommand{\cH}{\mathcal{H}}
\renewcommand{\v}{\boldsymbol{v}}
\newcommand{\w}{\boldsymbol{w}}
\newcommand{\z}{\boldsymbol{z}}
\renewcommand{\b}{\boldsymbol{b}}
\newcommand{\T}{\top}
\newcommand{\blue}[1]{\textbf{\color{tableau_blue}\textbf{{#1}}}}

\newcommand{\red}[1]{\textbf{\color{tableau_red}{#1}}}
\newcommand{\orange}[1]{\textbf{\color{tableau_orange}\textbf{{#1}}}}
\usepackage{placeins}
\newcommand{\ANC}[1]{$\overrightarrow{(\text{NC{#1})}}$}

\renewcommand{\eqref}[1]{[\ref{#1}]}

\newcommand{\revision}[1]{\hl{#1}}
\renewcommand{\revision}[1]{{#1}} 

\templatetype{pnasresearcharticle} 

\title{Prevalence of Neural Collapse during the terminal phase of deep learning training}

\author[a,1]{Vardan Papyan} 
\author[b,1]{X.Y. Han}
\author[a,2]{David L. Donoho}

\affil[a]{Department of Statistics, Stanford University}
\affil[b]{School of Operations Research and Information Engineering, Cornell University}

\leadauthor{Donoho} 

\significancestatement{Modern deep neural networks for image classification have achieved super-human performance. Yet, the complex details of trained networks have forced most practitioners and researchers to regard them as blackboxes with little that could be understood. This paper considers in detail a now-standard training methodology: driving the cross-entropy loss to zero, continuing long after the classification error is already zero. Applying this methodology to an authoritative collection of standard deepnets and datasets, we observe the emergence of a simple and highly symmetric geometry of the deepnet features and of the deepnet classifier; and we document important benefits that the geometry conveys -- thereby helping us understand an important component of the modern deep learning training paradigm.}

\authorcontributions{V.P. and X.H. performed the experiments, formulated theorems, and derived proofs. D.D. guided the experimental design, formulated theorems, and derived proofs.}
\equalauthors{\textsuperscript{1} V.P. and X.H. contributed equally to this work.}
\correspondingauthor{\textsuperscript{2}To whom correspondence should be addressed. E-mail: \texttt{donoho@stanford.edu}}


\begin{abstract}
Modern practice for training classification deepnets involves a \textit{Terminal Phase of Training} (TPT), which begins at the epoch where training error first vanishes; During TPT, the training error stays effectively zero while training loss is pushed towards zero.
Direct measurements of TPT, for three prototypical deepnet architectures and across seven canonical classification datasets, expose a pervasive     inductive bias we call \textit{Neural Collapse}, involving four deeply interconnected phenomena:  (NC1) Cross-example within-class variability of last-layer training activations collapses to zero, as the individual activations themselves collapse to their class-means; (NC2) The class-means collapse to the vertices of a Simplex Equiangular Tight Frame (ETF); (NC3) Up to rescaling, the last-layer classifiers collapse to the class-means, or in other words to the Simplex ETF, i.e. to a \textit{self-dual} configuration; (NC4) For a given activation, the classifier's decision collapses to simply choosing whichever class has the closest train class-mean, i.e. the Nearest Class-Center (NCC) decision rule.
The symmetric and very simple geometry induced by the TPT confers important benefits, including better generalization performance, better robustness, and better interpretability.
\end{abstract}

\dates{This manuscript was compiled on \today}
\doi{\url{www.pnas.org/cgi/doi/10.1073/pnas.XXXXXXXXXX}}

\begin{document}

\maketitle
\thispagestyle{firststyle}
\ifthenelse{\boolean{shortarticle}}{\ifthenelse{\boolean{singlecolumn}}{\abscontentformatted}{\abscontent}}{}


\section{Introduction}\label{sec:introduction}
Over the last decade, deep learning systems have steadily advanced the state-of-the-art in benchmark competitions, culminating in super-human performance in tasks ranging from image classification to language translation to game play. One might expect the trained networks to exhibit many particularities--making it impossible to find any empirical regularities across a wide range of datasets and architectures. On the contrary, in this article we present extensive measurements across image-classification datasets and architectures, exposing a common empirical pattern. 

Our observations focus on today's standard training paradigm in deep learning, an accretion of several fundamental ingredients that developed over time: Networks are trained beyond zero misclassification error, approaching negligible cross-entropy loss, \textit{interpolating} the in-sample training data; networks are \textit{overparameterized}, making such memorization possible; and these parameters are layered in ever-growing \textit{depth}, allowing for sophisticated feature engineering. A series of recent works \cite{ma2018power,belkin2019does,belkin2018overfitting,belkin2019reconciling,belkin2019talk} highlighted the paradigmatic nature of the practice of training well beyond zero-{\it error}, seeking zero-{\it loss}. We call the post-zero-error phase the \textbf{Terminal Phase of Training (TPT)}.

A scientist with standard preparation in mathematical statistics might anticipate that the linear classifier resulting from this paradigm, being a by-product of such training, would be quite arbitrary and vary wildly--from instance to instance, dataset to dataset, and architecture to architecture--thereby displaying no underlying cross-situational invariant structure. The scientist might further expect that the configuration of the fully-trained decision boundaries -- and the underlying linear classifier defining those boundaries -- would be quite arbitrary and vary chaotically from situation to situation. Such expectations might be supported by appealing to the overparameterized nature of the model, and to standard arguments whereby any noise in the data propagates during overparameterized training to generate disproportionate changes in the parameters being fit. 

Defeating such expectations, we show here that TPT frequently induces an underlying mathematical simplicity to the trained deepnet model -- and specifically to  the classifier and last-layer activations -- across many situations now considered canonical in deep learning. Moreover, the identified structure naturally suggests performance benefits. And indeed, we show that convergence to this rigid structure  tends to occur simultaneously with improvements in the  network’s generalization performance as well as adversarial robustness.

We call this process \textbf{Neural Collapse}, and characterize it by four manifestations in the classifier and last-layer activations:
\begin{description}
    \item[(NC1) Variability collapse:] As training progresses, the within-class variation of the activations becomes negligible as these activations collapse to their class-means.
    \item[(NC2) Convergence to Simplex ETF:] The vectors of the class-means (after centering by their global-mean) converge to having equal length, forming equal-sized angles between any given pair, and being the maximally pairwise-distanced configuration constrained to the previous two properties. This configuration is identical to a previously studied configuration in the mathematical sciences known as Simplex \textbf{Equiangular Tight Frame (ETF)} \cite{strohmer2003grassmannian}. See Definition \ref{def:etf}.
    \item[(NC3) Convergence to self-duality:] The class-means and linear classifiers -- although mathematically quite different objects, living in dual vector spaces -- converge to each other, up to rescaling. \revision{Combined with (NC2), this implies a \textit{complete} \textit{symmetry} in the network classifiers' decisions: each iso-classifier-decision region is isometric to any other such region by rigid Euclidean motion; moreover the class-means are each centrally located within their own specific regions, so there is no tendency towards higher confusion between any two classes than any other two.}
    \item[(NC4) Simplification to Nearest Class-Center (NCC):] For a given deepnet activation, the network classifier converges to choosing whichever class has the nearest train class-mean (in standard Euclidean distance).
\end{description}
We give a visualization of the phenomena (NC1)-(NC3) in Figure \ref{fig:etf}\footnote{Figure \ref{fig:etf} is, in fact, generated using \textit{real measurements}, collected while training the VGG13 deepnet on CIFAR10: For three randomly selected classes, we extract the linear classifiers, class-means, and a subsample of twenty last-layer features \ETFfig{at epochs 2, 16, 65, and 350}{}. These entities are then rotated, rescaled, and represented in three-dimensions by leveraging the singular-value decomposition of the class-means. We omit further details as Figure \ref{fig:etf} serves only to illustrate Neural Collapse on an abstract level.}, and define Simplex ETFs (NC2) more formally as follows:
\begin{definition}[\textbf{Simplex ETF}]\label{def:etf}
A \textit{standard} Simplex ETF is a collection of points in $\R^C$ specified by the columns of 
\begin{equation}\label{eq:etf_def}
    \M^\star = \sqrt{\frac{C}{C-1}} \left( \I - \frac{1}{C} \one\one^\T \right),
\end{equation}
where $\I \in \R^{C \times C}$ is the identity matrix, and $\one_C \in \R^C$ is the ones vector. In this paper, we allow other poses, as well as rescaling, so the \textit{general} Simplex ETF consists of the points specified by the columns of $\M = \alpha \U \M^\star \in \R^{p \times C}$, where $\alpha \in \R_{+}$ is a scale factor, and $\U \in \R^{p\times C}$ ($p \geq C$) is a partial orthogonal matrix ($\U^\T \U = \I$).
\end{definition}

\begin{figure}[tbhp]
\centering
\includegraphics[width=0.39\textwidth]{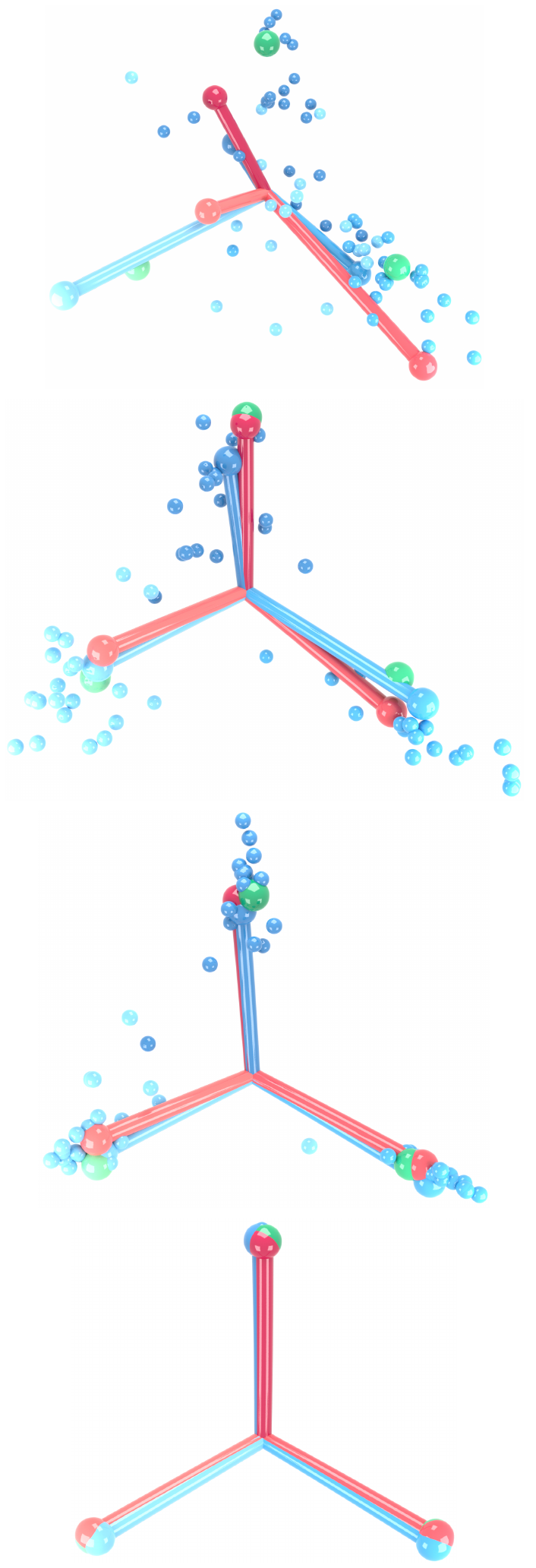}
\caption{\textbf{Visualization of Neural Collapse:} 
The figures depict, in three dimensions, Neural Collapse as training proceeds, from top to bottom. Green spheres represent the vertices of the standard Simplex ETF (Definition \ref{def:etf}), red ball-and-sticks represent linear classifiers, blue ball-and-sticks represent class-means, and small blue spheres represent last-layer features. For all objects, we distinguish different classes via the shade of the color. As training proceeds, last-layer features collapse onto their class-means (NC1), class-means converge to the vertices of the Simplex ETF (NC2), the linear classifiers approach their corresponding class-means (NC3). An animation can be found \href{https://purl.stanford.edu/br193mh4244}{here}. \label{fig:etf}}
\end{figure}

\revision{
Properties (NC1)-(NC4) show that a highly symmetric and rigid mathematical structure with clear interpretability arises spontaneously during deep learning feature engineering, identically across many different datasets and model architectures.}

\revision{(NC2) implies that the different feature means are `equally spaced' around the sphere in their constructed feature space; (NC3) says the same for the linear classifiers in their own dual space; and moreover, that the linear classifiers are `the same as' the class means, up to possible rescaling.  These mathematical symmetries and rigidities vastly simplify the behavior and analysis of trained classifiers, as we show in Section} \ref{sec:implications} \revision{below, which contrasts the kind of qualitative understanding previously available from theory, against the precise and highly constrained predictions possible with (NC4).

(NC1)-(NC4) offer theoretically-established performance benefits: stability against random noise and against adversarial noise. And indeed, this theory bears fruit. We show that during TPT, while Neural Collapse is progressing, the trained models are improving in generalizability and in adversarial robustness.}

\revision{In Section} \ref{sec:related_works} \revision{below we discuss the broader significance of (NC1)-(NC4) and their relation to recent advances across several rapidly developing  `research fronts.'}

\revision{To support our conclusions, we conduct} empirical studies that range over seven canonical classification datasets \cite{krizhevsky2009learning,lecun2010mnist,xiao2017fashion}, including ImageNet \cite{deng2009imagenet}, and three prototypical, contest-winning architectures \cite{he2016deep,huang2017densely,simonyan2014very}. These datasets and networks were chosen for their prevalence in the literature as benchmarks \cite{benchmarkAI,stanfordDAWN,rodrigob}, reaffirmed by their easy availability as part of the popular deep learning framework PyTorch \cite{paszke2019pytorch}. As explained below, these observations have important implications  for our understanding of numerous theoretical and empirical observations in deep learning.

\section{Setting and methodology}\label{sec:exper_info}
All subsequent experiments are built upon the general setting and methodology described below.

\subsection{Image classification}
In the image classification problem, we are given a training dataset of $d$-dimensional images, the goal is to train a predictor to identify the class -- out of $C$ total classes -- to which any input image $\x \in \R^d$ belongs.

\subsection{Deep learning for classification}
In this work, we consider the predictor to be a deep neural network, which typically consists of numerous layers followed by a linear classifier. We view the layers before the classifier as computing a function, $\x \to \h(\x)$, where $\h:\R^d \to \R^p$ outputs a $p$-dimensional feature vector. We refer to $\h(\x)$ as the \textit{last-layer activations} or \textit{last-layer features}.   The linear classifier takes as inputs the last-layer activations and outputs the class label.  In detail, the linear classifier is specified by weights $\W \in \R^{C \times p}$ and biases $\b \in \R^{C}$, and the label the network attaches to image $\x$ is a simple function of $\W \h(\x)+\b$. In fact, it is
$\arg\max_{c'} \left< \w_{c'}, \h \right> + b_{c'}$ i.e. the label is the index of the largest element in the vector $\W \h(\x)+\b$.

\subsection{Network architecture and feature engineering}
The network is generally specified in two stages: First, an architecture is prescribed; and then, for a given architecture, there are a large number of parameters which determine the deep network's feature engineering, $\h(\x)$. Collecting these parameters in a vector $\btheta$, we may also write $\h = \h_{\btheta}(\x)$.

When the architecture specifies a truly deep network -- and not merely a shallow one -- the variety of behaviors that the different choices of $\btheta$ can produce is quite broad. To evoke, in quite concrete terms, the process of specifying the nonlinear transformation $\x \mapsto \h_{\btheta}(\x)$, we speak of {\it  feature engineering}.  In contrast, traditional machine learning often dealt with a fixed collection of feature vectors that were not data-adaptive.

\subsection{Training}

Viewing the induced class labels as the network outputs, and the architecture and problem size as fixed in advance, the underlying class labelling algorithm depends on the parameter vector $(\btheta,\W,\b)$. We think of $\btheta$ as determining the features to be used, and $(\W,\b)$ as determining the linear classifier that operates on the features to produce the labels. The number of parameters that must be determined is quite large. In practice, these parameters must be learned from data, by the process commonly known as  training.  

More concretely, consider a balanced dataset, having exactly $N$ training examples in each class, $\bigcup_{c=1}^C \{ \x_{i,c} \}_{i=1}^N$, where $\x_{i,c}$ denotes the $i$-th example in the $c$-th class. The parameters $(\btheta, \W,\b)$ are fit by minimizing, usually using stochastic gradient descent (SGD), the objective function:
\begin{equation}\label{eq:train_loss}
    \min_{\btheta, \W, \b} \sum_{c=1}^C \sum_{i=1}^N \L \left(\W \h_{\btheta}(\x_{i,c}) + \b, \y_c \right).
\end{equation}
Above, we denote by $\L: \R^C{\times}\R^C \to \R^+$ the cross-entropy loss function and by $\y_c \in \R^C$ one-hot vectors, i.e, vectors containing one in the $c$-th entry and zero elsewhere. We refer to this quantity as the \textit{training loss} and the number of incorrect class predictions made by the network as the \textit{training error}. Notice that, in TPT, \textit{the loss is non-zero even if the classification error is zero}.

\subsection{Datasets}
We consider the MNIST, FashionMNIST, CIFAR10, CIFAR100, SVHN, STL10 and ImageNet datasets \cite{deng2009imagenet,krizhevsky2009learning,lecun2010mnist,xiao2017fashion}. MNIST was sub-sampled to $N{=}5000$ examples per class, SVHN to $N{=}4600$ examples per class, and ImageNet to $N{=}600$ examples per class. The remaining datasets are already balanced. The images were pre-processed, pixel-wise, by subtracting the mean and dividing by the standard deviation. No data augmentation was used.

\subsection{Networks}
We train the VGG, ResNet, and DenseNet architectures \cite{he2016deep,huang2017densely,simonyan2014very}.  For each of the three architecture types, we chose the network depth   through trial-and-error in a series of preparatory experiments in order to adapt to the varying difficulties of the datasets. The final chosen networks were VGG19, ResNet152, and DenseNet201 for ImageNet; VGG13, ResNet50, and DenseNet250 for STL10; VGG13, ResNet50, and DenseNet250 for CIFAR100; VGG13, ResNet18, and DenseNet40 for CIFAR10; VGG11, ResNet18, and DenseNet250 for FashionMNIST; and VGG11, ResNet18, and DenseNet40 for MNIST and SVHN. DenseNet201 and DenseNet250 were trained using the memory-efficient implementation proposed in \cite{pleiss2017memory}. We replaced the dropout layers in VGG with batch normalization and set the dropout rate in DenseNet to zero.

\subsection{Optimization methodology}\label{sec:methodology} Following common practice, we minimize the cross-entropy loss using stochastic gradient descent (SGD) with momentum $0.9$. The weight decay is set to $1{\times}10^{-4}$ for ImageNet and $5{\times}10^{-4}$ for the other datasets. ImageNet is trained with a batch size of $256$, across 8 GPUs, and the other datasets are trained on a single GPU with a batch size of $128$. We train ImageNet for 300 epochs and the other datasets for $350$ epochs. The initial learning is annealed by a factor of $10$ at $1/2$ and $3/4$ for ImageNet; and  $1/3$ and $2/3$ for the other the datasets. We sweep over $10$ logarithmically-spaced learning rates for ImageNet between $0.01$ and $0.25$, and $25$ learning rates for the remaining datasets between $0.0001$ and $0.25$--picking the model resulting in the best test error in the last epoch.

\subsection{Large-scale experimentation}
The total number of models fully trained for this paper is tallied below:

\begin{align*}
     \text{ImageNet: }     & \text{1 dataset} \times \text{3 nets} \times \text{10 lrs} = \text{30 models}.\\
     \text{Remainder: } & \text{6 datasets} \times \text{3 nets} \times \text{25 lrs} = \text{450 models}.\\
     \text{Total: } & \text{480 models}.
\end{align*}
The massive computational experiments reported here were run painlessly using ClusterJob and ElastiCluster \citep{clusterjob,MMCEP17,Monajemi19} on the Stanford Sherlock HPC cluster and Google Compute Engine virtual machines.

\subsection{Moments of activations} During training, we snapshot the network parameters at certain epochs. For each snapshotted epoch, we pass the train images through the network, extract their last-layer activations (using PyTorch hooks \cite{pytorch:hooks}), and calculate these activations' first and second moment statistics.

For a given dataset-network combination, we calculate the train global-mean $\bmu_G \in \R^p$:

\begin{equation*}
    \bmu_G \triangleq \Ave_{i,c} \{ \h_{i,c} \},
\end{equation*}
and the train class-means $\bmu_c \in \R^p$:

\begin{equation*}
    \bmu_c \triangleq \Ave_{i} \{ \h_{i,c} \}, \quad c=1,\dots,C,
\end{equation*}
where $\Ave$ is the averaging operator.

Unless otherwise specified, for brevity, we refer in the text to the globally-centered class-means, $\{\bmu_c - \bmu_G\}_{c=1}^C$, as just \textit{class-means}, since the globally-centered class-means are of more interest.


Given the train class-means, we calculate the train total covariance $\St \in \R^{p \times p}$,

\begin{equation*}
    \St \triangleq \Ave_{i,c} \left\{ \left( \h_{i,c} - \bmu_G \right) \left( \h_{i,c} - \bmu_G \right)^\T \right\},
\end{equation*}
the between-class covariance, $\Sb \in \R^{p \times p}$,
\begin{equation}\label{eq:SB_def}
\begin{aligned}
    \Sb \triangleq & \Ave_{c} \{ (\bmu_c - \bmu_G) (\bmu_c - \bmu_G)^\T \},
\end{aligned}
\end{equation}
and the within-class covariance, $\Sw \in \R^{p \times p}$,

\begin{align}
    \Sw \triangleq & \Ave_{i,c} \{ (\h_{i,c} - \bmu_{c}) (\h_{i,c} - \bmu_{c})^\T \}.
\end{align}
Recall from multivariate statistics that:

\begin{equation*}
    \St = \Sb + \Sw.
\end{equation*}

\subsection{Formalization of Neural Collapse}

With the above notation, we now present a more mathematical description of Neural Collapse, where $\to$ indicates convergence as training progresses:

\begin{description}
    \item[(NC1) Variability collapse:] $\Sw \to \zr$
    \item[(NC2) Convergence to Simplex ETF:] 
        \begin{gather*}
            \left| \left\|\bmu_c - \bmu_G \right\|_2 - \left\|\bmu_{c'} - \bmu_G\right\|_2 \right| \to  0 \quad \forall \ c,c'\\
            \left \langle \tilde{\bmu}_c,\tilde{\bmu}_{c'} \right \rangle \to  \frac{C}{C-1}\delta_{c,c'} - \frac{1}{C-1} \quad \forall \ c,c'.
        \end{gather*}
    \item [(NC3) Convergence to self-duality:]
    \begin{equation}
        \left\| \frac{\W^\T}{\|\W\|_F} - \frac{\Mc}{\|\Mc\|_F}\right\|_F \to 0
    \end{equation}
    \item [(NC4): Simplification to NCC:] 
    \begin{equation*}
    \arg\max_{c'} \left< \w_{c'}, \h \right> + b_{c'} \to \argmin_{c'} \|\h - \bmu_{c'}\|_2
    \end{equation*}
\end{description}
where $\tilde{\bmu}_c = (\bmu_c - \bmu_G)/\|\bmu_c - \bmu_G\|_2$ are the renormalized the class-means, $\Mc = [\bmu_c - \bmu_G, c=1,\dots,C] \in \R^{p \times C}$ is the matrix obtained by stacking the class-means into the columns of a matrix, and $\delta_{c,c'}$ is the Kronecker delta symbol.

\section{Results}\label{sec:figure_description}
To document the observations we make in Section \ref{sec:introduction}, we provide a series of figures and tables below. We briefly list here our claims and identify the source of our evidence.
\begin{itemize}
    \item Means and classifiers become equinorm: Figure \ref{fig:equinorm_std}
    \item Means and classifiers become maximally equiangular: \\ Figures \ref{fig:equi_angle} and \ref{fig:max_angle}
    \item Means and classifiers become self-dual: Figure \ref{fig:w_minus_h}
    \item Train within-class covariance collapses: Figure \ref{fig:collapse}
    \item Classifier approaches nearest class-center: Figure \ref{fig:nn_match}
    \item TPT improves robustness: Figure \ref{fig:adversarial}
    \item TPT improves test-error: Table \ref{tab:generalization}
\end{itemize}

All figures in this article are formatted as follows: Each of the seven array columns is a canonical dataset for benchmarking classification performance -- ordered left to right roughly by ascending difficulty. Each of the three array rows is a prototypical deep classifying network. On the horizontal axis of each cell is the epoch of training.  For each dataset-network combination, the \red{red} vertical line marks the begining of the effective beginning of TPT, i.e, the epoch when the training accuracy reaches 99.6\% for ImageNet and 99.9\% for the remaining datasets; we do not use 100\% as it has been reported \cite{ekambaram2017finding,zhang2018improved,muller2019identifying} that several of these datasets contain inconsistencies and mislabels which sometimes prevent absolute memorization. Additionally, \orange{orange} lines denote measurements on the network \orange{classifier}, while \blue{blue} lines denote measurements on the activation \blue{class-means}.


\begin{figure*}[tbhp]
\centering
\includegraphics[width=\linewidth]{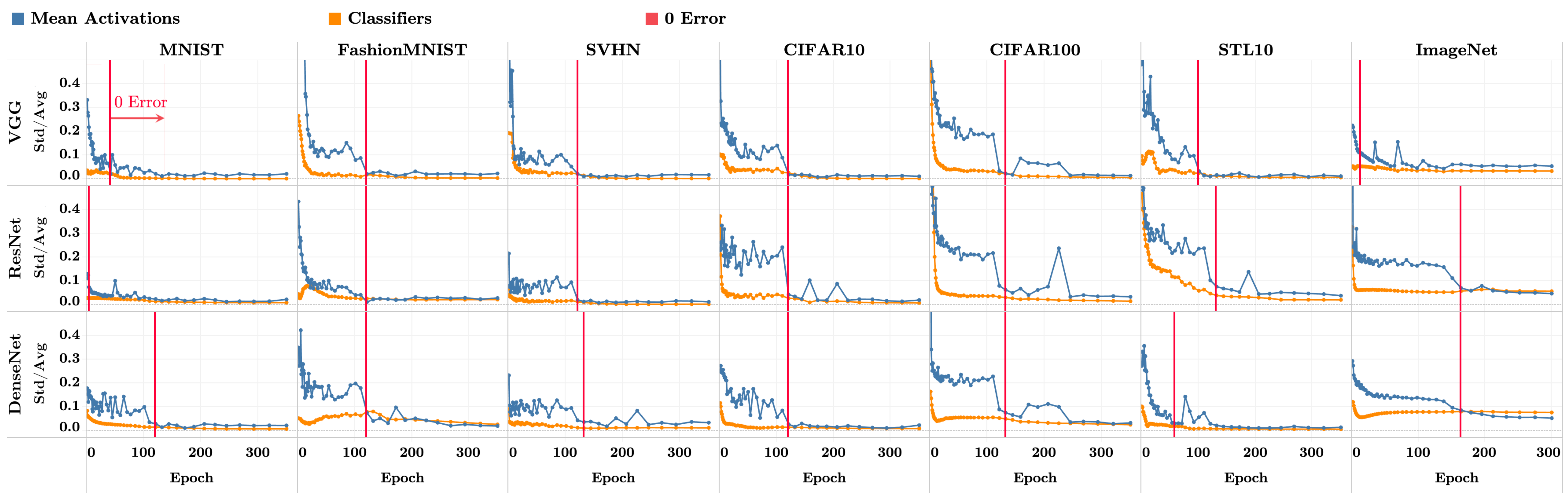}
\caption{\textbf{Train class-means become equinorm:} \figstart 
In each array cell, the vertical axis shows the coefficient of variation of the centered class-mean norms as well as the network classifiers norms. In particular, the \blue{blue} line shows $\text{Std}_c(\|\bmu_c - \bmu_G\|_2)/\text{Avg}_c(\|\bmu_c - \bmu_G\|_2)$ where $\{\bmu_c\}$ are the class-means of the last-layer activations of the training data and $\bmu_G$ is the corresponding train global-mean; the \orange{orange} line shows $\text{Std}_c(\|\w_c\|_2)/\text{Avg}_c(\|\w_c\|_2)$ where $\w_c$ is the last-layer classifier of the $c$-th class. 
As training progresses, the coefficients of variation of both class-means and classifiers decreases.
}
\label{fig:equinorm_std}
\end{figure*}

\begin{figure*}[tbhp]
\centering
\includegraphics[width=\linewidth]{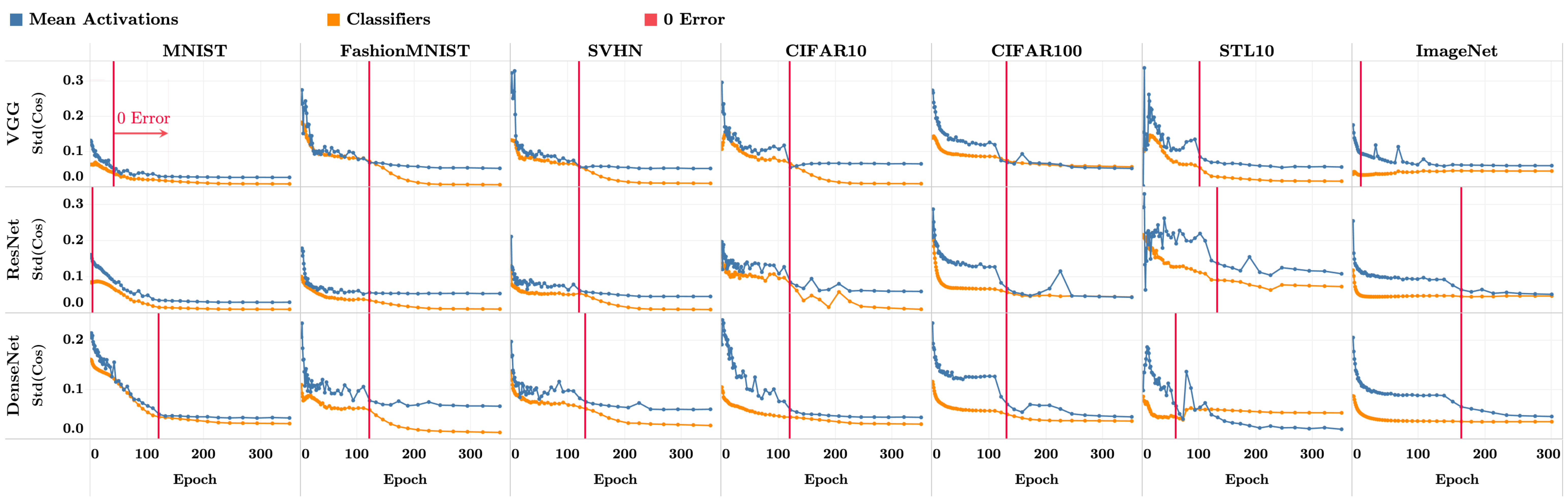}
\caption{\textbf{Classifiers and train class-means approach equiangularity:} \figstart
In each array cell, the vertical axis shows the standard deviation of the cosines between pairs of centered class-means and classifiers across all distinct pairs of classes $c$ and $c^\prime$. Mathematically, denote $\cos_{\bmu}(c,c^\prime)=\left\langle \bmu_{c}-\bmu_{G},\bmu_{c^\prime}-\bmu_{G}\right\rangle/( \|\bmu_{c}-\bmu_{G}\|_2 \|\bmu_{c^\prime}-\bmu_{G}\|_2$ and $\cos_{\w} (c,c^\prime)=\left\langle \w_c,\w_{c^\prime}\right\rangle/( \|\w_c\|_2 \|\w_{c^\prime}\|_2)$ where $\{\w_c\}_{c=1}^C$, $\{\bmu_c\}_{c=1}^C$, and $\bmu_G$ are as in Figure \ref{fig:equinorm_std}. We measure $\text{Std}_{c, c^\prime \neq c} (\cos_{\bmu} (c,c^\prime))$ (\blue{blue}) and  $\text{Std}_{c, c^\prime \neq c} (\cos_{\w} (c,c^\prime))$ (\orange{orange}). As training progresses, the standard deviations of the cosines approach zero indicating equiangularity.
}
\label{fig:equi_angle}
\end{figure*}

\begin{figure*}[tbhp]
\centering
\includegraphics[width=\linewidth]{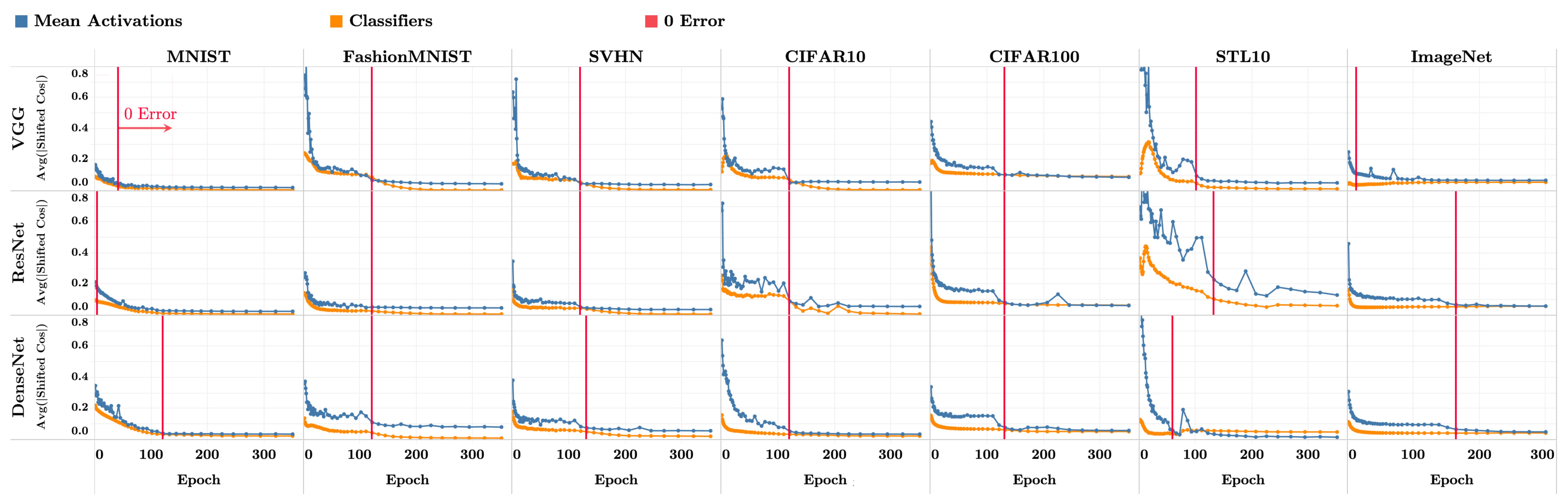}
\caption{\textbf{Classifiers and train class-means approach maximal-angle equiangularity:} \figstart 
We plot in the vertical axis of each cell the quantities $\text{Avg}_{c,c^\prime} |\cos_{\bmu}(c,c^\prime) + 1/(C-1)|$ (\blue{blue}) and $\text{Avg}_{c,c^\prime} |\cos_{\w}(c,c^\prime) + 1/(C-1)|$ (\orange{orange}), where $\cos_{\bmu}(c,c^\prime)$ and $\cos_{\w}(c,c^\prime)$ are as in Figure \ref{fig:equi_angle}. As training progresses, the convergence of these values to zero implies that all cosines converge to $-1/(C-1)$. This corresponds to the maximum separation possible for globally centered, equiangular vectors.
}
\label{fig:max_angle}
\end{figure*}

\begin{figure*}[tbhp]
\centering
\includegraphics[width=\linewidth]{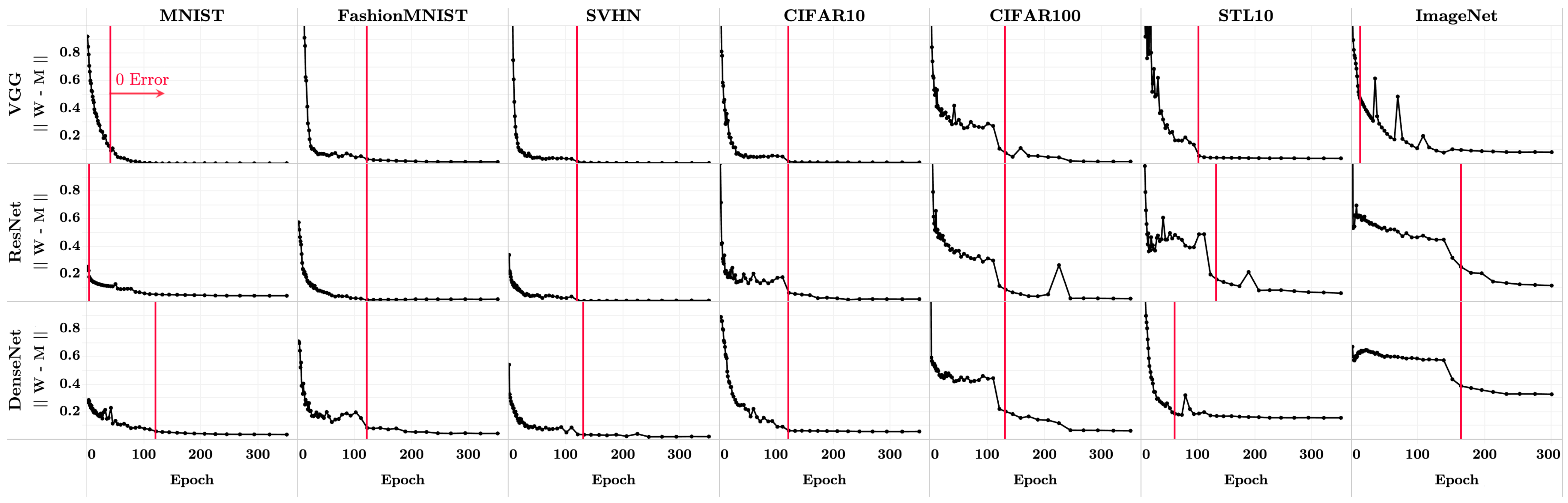}
\caption{\textbf{Classifier converges to train class-means:} \figstart 
In the vertical axis of each cell, we measure the distance between the classifiers and the centered class-means, both rescaled to unit-norm. Mathematically, denote $\widetilde{\M} = \Mc / \|\Mc\|_F$ where $\Mc = [\bmu_c - \bmu_G: c=1,\dots,C] \in \R^{p \times C}$  is the matrix whose columns consist of the centered train class-means; denote $\widetilde{\W} = \W / \|\W\|_F$ where $\W \in \R^{C \times p}$ is the last-layer classifier of the network. We plot the quantity $\|\widetilde{\W}^\T - \widetilde{\M}\|^2_F$ on the vertical axis. This value decreases as a function of training, indicating the network classifier and the centered-means matrices become proportional to each other (self-duality).
}
\label{fig:w_minus_h}
\end{figure*}

\begin{figure*}[tbhp]
\centering
\includegraphics[width=\linewidth]{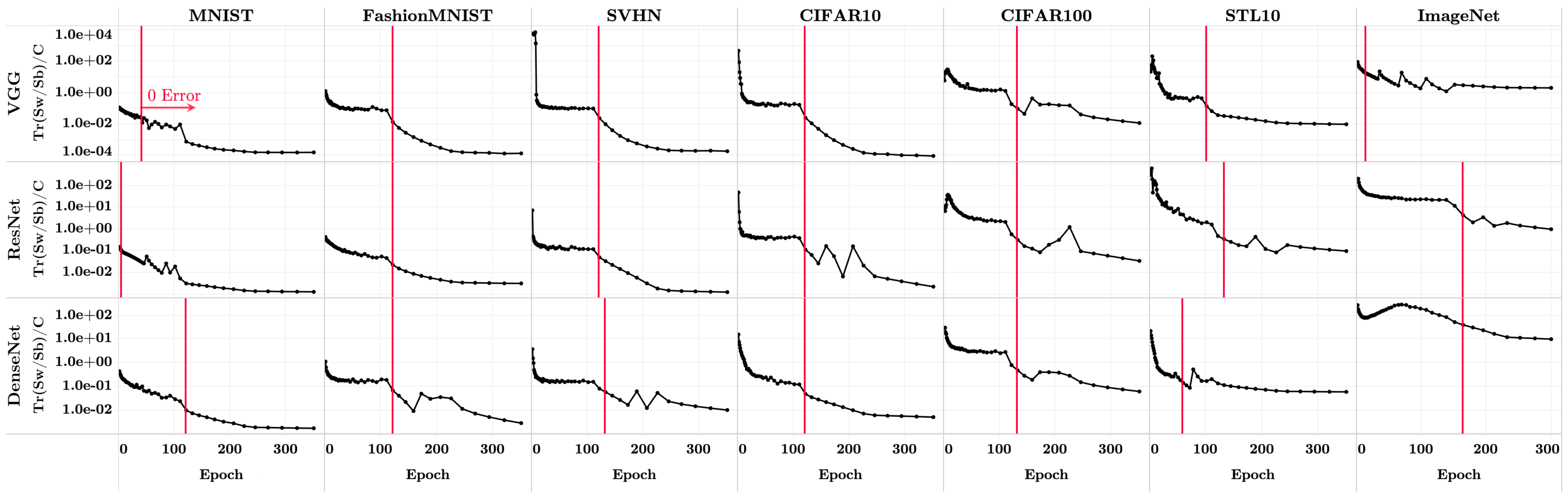}
\caption{\textbf{Training within-class variation collapses:} \figstart 
In each array cell, the vertical axis (log-scaled) shows the magnitude of the between-class covariance compared to the within-class covariance of the train activations . Mathematically, this is represented by $\Tr{\Sw\Sb^{\dagger}}/C$ where $\Tr{\cdot}$ is the trace operator, $\Sw$ is the within-class covariance of the last-layer activations of the training data, $\Sb$ is the corresponding between-class covariance, $C$ is the total number of classes, and $[\cdot]^{\dagger}$ is Moore-Penrose pseudoinverse. This value decreases as a function of training -- indicating collapse of within-class variation.
}
\label{fig:collapse}
\end{figure*}

\begin{figure*}[tbhp]
\centering
\includegraphics[width=\linewidth]{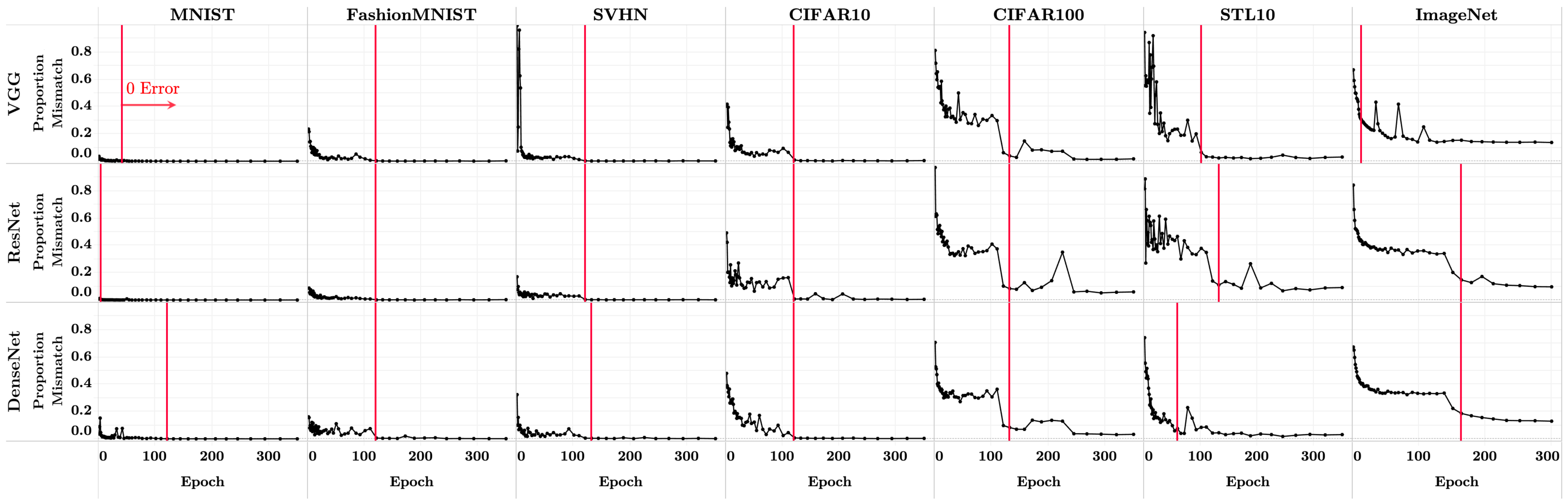}
\caption{\textbf{Classifier behavior approaches that of Nearest Class-Center:} \figstart 
In each array cell, we plot the proportion of examples (vertical axis) in the \textit{testing} set where network classifier disagrees with the result that would have been obtained by choosing  $\arg\min_c \|\h - \bmu_{c}\|_2$ where $\h$ is a last-layer test activation, and $\{\bmu_c\}_{c=1}^C$ are the class-means of the last-layer train activations. As training progresses, the disagreement tends to zero, showing the classifier's behavioral simplification to the nearest train class-mean decision rule.
}
\label{fig:nn_match}
\end{figure*}

\begin{figure*}[tbhp]
\centering
\includegraphics[width=\linewidth]{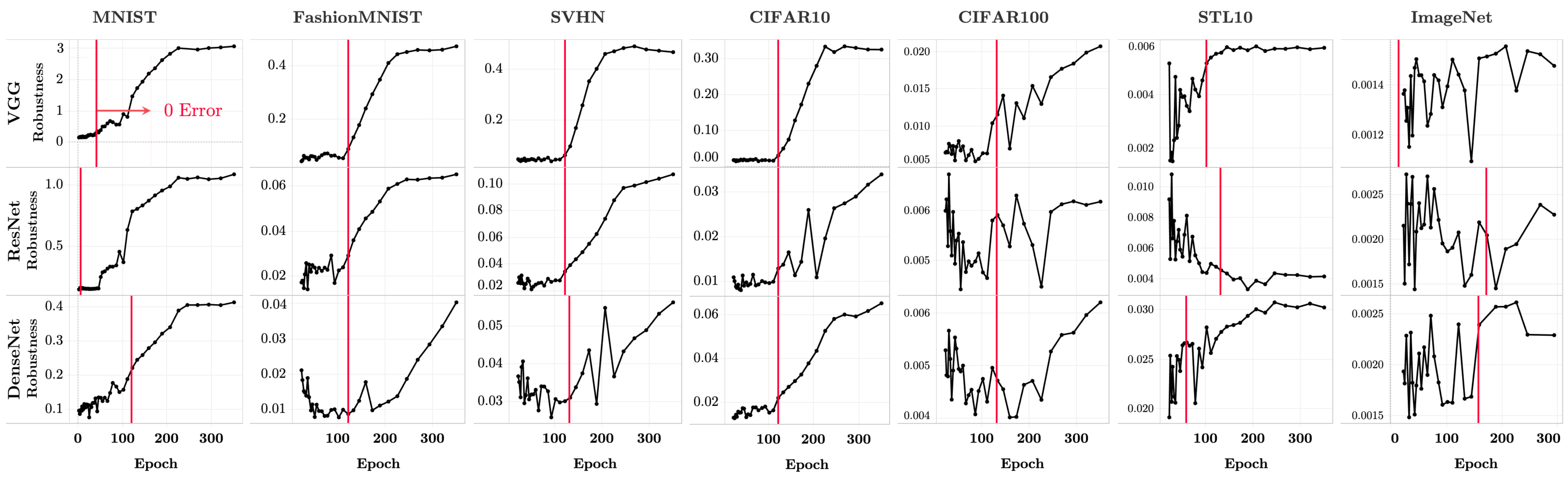}
\caption{\textbf{Training beyond zero-error improves adversarial robustness:} \figstart 
For each dataset and network, we sample without replacement $100$ test images--constructing for each an adversarial example using the DeepFool method proposed in \cite{moosavi2016deepfool}. In each array cell, we plot on the vertical axis the robustness measure, $\Ave_i \|\r(\x_i)\|_2 / \|\x_i\|_2$, from the same paper--where $\r(\x_i)$ is the minimal perturbation required to change the class predicted by the classifier, for a given input image $\x_i$. As training progresses, larger perturbations are required to fool the deepnet. Across all array cells, the median improvement in the robustness measure in the last epoch over the first epoch achieving zero training error is 0.0252; the mean improvement is 0.2452.
}
\label{fig:adversarial}
\end{figure*}

\newcolumntype{Z}{>{\centering\arraybackslash}m{1.8cm}}
\newcolumntype{L}{>{\centering\arraybackslash}m{1.8cm}}

\begin{table} [t] 
\centering
\caption{\textbf{Training beyond zero-error improves test performance.}
}\label{tab:generalization}
\begin{tabular}{llZL} 
\textbf{Dataset} & \textbf{Network} & \textbf{Test accuracy at zero error} & \textbf{Test accuracy at last epoch} \\ \toprule  & VGG & 99.40 & 99.56 \\ MNIST & ResNet & 99.32 & 99.71 \\  & DenseNet & 99.65 & 99.70 \\ \midrule  & VGG & 92.92 & 93.31 \\ FashionMNIST & ResNet & 93.29 & 93.64 \\  & DenseNet & 94.18 & 94.35 \\ \midrule  & VGG & 93.82 & 94.53 \\ SVHN & ResNet & 94.64 & 95.70 \\  & DenseNet & 95.87 & 95.93 \\ \midrule  & VGG & 87.85 & 88.65 \\ CIFAR10 & ResNet & 88.72 & 89.44 \\  & DenseNet & 91.14 & 91.19 \\ \midrule  & VGG & 63.03 & 63.85 \\ CIFAR100 & ResNet & 66.19 & 66.21 \\  & DenseNet & 77.19 & 76.56 \\ \midrule  & VGG & 65.15 & 68.00 \\ STL10 & ResNet & 69.99 & 70.24 \\  & DenseNet & 67.79 & 70.81 \\ \midrule  & VGG & 47.26 & 50.12 \\ ImageNet & ResNet & 65.41 & 64.45 \\  & DenseNet & 65.04 & 62.38 \\ \bottomrule\end{tabular}
\justify
\addtabletext{The median improvement of test accuracy at the last epoch over that at the first epoch achieving zero training error is 0.3495 percent; the mean improvement is 0.4984 percent.}
\end{table}

\section{Discussion} \label{sec:discussion}
Taken together, Figures \ref{fig:equinorm_std}-\ref{fig:nn_match} give evidence for Neural Collapse. First, Figure \ref{fig:equinorm_std} shows how, as training progresses, the variation in the norms of the class-means (and classifiers) decreases--indicating that the class-means (and classifiers) are converging to an equinormed state.

Then, Figure \ref{fig:equi_angle} indicates that all pairs of class-means (or classifiers) tend towards forming equal-sized angles. Figure \ref{fig:max_angle} additionally reveals that the cosines of these angles converge to $-\frac{1}{C-1}$ -- the maximum possible given the constraints. This maximal-equiangularity, combined with equinormness, implies that the class-means and classifiers converge to Simplex ETFs.

The above experiments by themselves do not indicate any relationship between the final converged states of the class-means and classifiers, even though both converge to some Simplex ETF. Such a relationship is revealed by Figure \ref{fig:w_minus_h} -- showing how they converge to the \textit{same} Simplex ETF, up to rescaling.

Moreover, the above concerns solely the class-means. \revision{Yet, we can make a stronger claim about the activations themselves by looking at $\tr{\Sw \Sb^{\dagger}}$. This quantity, canonical in multivariate statistics, measures the \textit{inverse} \textit{signal-to-noise} ratio for classification problems and can be used to predict misclassification} \cite{anderson1962introduction} \revision{. The intra-class covariance matrix $\Sw$ (the noise) is best interpreted once scaled and rotated by pseudo-inverse of the inter-class covariance matrix $\Sb$ (the signal), since such transformation maps the noise into a common frame of reference across epochs. According to Figure} \ref{fig:collapse}  \revision{, the normalized variation of the activations becomes negligible as training proceeds, indicating the activations collapse to their corresponding class means. This collapse continues well after the beginning of TPT.}

A recurring theme across Figures \ref{fig:equinorm_std}-\ref{fig:nn_match} is the continuing process of Neural Collapse after zero error has been achieved. This explains TPT's paradigmatic nature: While continuing training after zero error has already been achieved seems counter-intuitive, it induces significant changes in the underlying structure of the trained network.

This motivates Figure \ref{fig:adversarial} and Table \ref{tab:generalization}, which explore two of the benefits of TPT. Table \ref{tab:generalization} shows how the test accuracy continues improving steadily, and Figure \ref{fig:adversarial} shows how adversarial robustness continues improving as well. In fact, most of the improvement in robustness happens \textit{during} TPT.






\section{Neural Collapse sharpens previous insights}\label{sec:implications}

Two notable prior works \cite{webb1990optimised,soudry2018implicit} were able to significantly constrain the form of the structure of trained classifiers. However, in the presence of Neural Collapse, it is possible to say dramatically more about the structure of trained classifiers. Moreover, the structure which emerges is extremely simple and symmetric.

In particular, the prior work assumed \textit{fixed} features, {\it not subject to data-adaptive feature engineering} which is so characteristic of deep learning training. In the modern context where {\it deep-learning features are trained}, and employing the assumption that the resulting last-layer entities undergo Neural Collapse, the mathematical constraints on the structure of trained classifiers tighten drastically.

As a preliminary step, we first formalize four possible properties of the \textit{end-state} towards which Neural Collapse is tending:
\begin{description}
    \item[\ANC{1} Variability collapse:] $\Sw = \zr$
    \item[\ANC{2} Simplex ETF structure:] 
    \begin{align*}
        \|\bmu_c - \bmu_G\|_2 =& \|\bmu_{c'} -\bmu_G \|_2 \quad \forall \ c,c'\\
        \left \langle \tilde{\bmu}_c,\tilde{\bmu}_{c'} \right \rangle =&  \frac{C}{C-1}\delta_{c,c'} - \frac{1}{C-1}.
    \end{align*}
    \item [\ANC{3} Self-duality:]  $\frac{\W^\T}{\|\W\|_F} = \frac{\Mc}{\|\Mc\|_F}$
    \item [\ANC{4} Behavioral equivalence to NCC:] 
    \begin{equation*}
    \arg\max_{c'} \left< \w_{c'}, \h \right> + b_{c'} = \argmin_{c'} \|\h - \bmu_{c'}\|_2.
    \end{equation*}
\end{description}
where $\tilde{\bmu}_c = (\bmu_c - \bmu_G)/\|\bmu_c - \bmu_G\|_2$ are the renormalized and centered class-means, $\Mc = [\mu_c - \mu_G: c=1,\dots,C]\in \R^{p \times C}$ is the matrix obtained by stacking the centered class-means into the columns of a matrix, and $\delta_{c,c'}$ is the Kronecker delta symbol.

\subsection{Webb and Lowe (1990)}\label{sec:webb_and_lowe}
In \cite{webb1990optimised}, Webb and Lowe proved the following important result which, reformulated for the modern setting, could be written as follows.

\begin{proposition}[Section 3 in \cite{webb1990optimised}] \label{prop:web_low}
    Fix the deepnet architecture and the underlying tuning parameters $\btheta$, so that the activations $\h_{\btheta}(\x)$ involve no training, and so that only the classifier weights $\W$ and biases $\b$ need to be trained. Maintain the same definitions -- $\St$, $\bmu_c$, $\bmu_{G}$ etc. -- as in Section \ref{sec:exper_info}. Adopting the mean squared error loss in place of the cross-entropy loss, the optimal classifier weights and biases are given by
    
    \begin{align}\label{eq:lda}
        \W = & \frac{1}{C} \Mc^\T \St^\dagger \nonumber \\
        \b = & \frac{1}{C} \one_C - \frac{1}{C} \Mc^\T \St^\dagger \bmu_G,
    \end{align}
    where $^{\dagger}$ denotes the Moore-Penrose pseudoinverse, $\Mc = [\bmu_c - \bmu_G: c=1,\dots,C] \in \R^{p \times C}$ is the matrix obtained by stacking the centered class-means into the columns of a matrix, and $\one_C \in \R^C$ is a vector of ones.
\end{proposition}

The form in \eqref{eq:lda} is similar to the one first developed by R.A. Fisher in 1936 \cite{fisher1936use} -- commonly referred to as Linear Discriminant Analysis (LDA) -- although Fisher's version uses $\Sw$ in lieu of $\St$. In other words, the above theorem states that a modified LDA is the optimal solution for the last-layer classifier

Webb and Lowe's result admirably elucidates the structure of the optimal classifier; however, it also leaves a great deal unspecified about possible properties of the classifier.
In our Theorem \ref{thm:collapse_web_low}, immediately following, we supplement Webb and Lowe's assumptions by adding the variability collapse and Simplex ETF properties; the result significantly narrows the possible structure of optimal classifiers, obtaining both self duality and behavioral agreement with NCC.

\begin{theorem} [\bf Prop.~\ref{prop:web_low}+\ANC{1-2} imply \ANC{3-4}] \label{thm:collapse_web_low}
Adopt the framework and assumptions of Proposition \ref{prop:web_low}, as well as the end state implied by (NC1)-(NC2), i.e. \ANC{1}-\ANC{2}. The Webb-Lowe classifier \eqref{eq:lda}, in this setting, has the additional properties \ANC{3}-\ANC{4}.
\end{theorem}
\begin{proof}
By \ANC{1}, $\Sw = \zr$, and we have $\St = \Sb$. Using now Proposition \ref{prop:web_low}, we obtain

\begin{align*}
    \W = & \frac{1}{C} \Mc^\T \Sb^\dagger \nonumber \\
    \b = & \frac{1}{C} \one_C - \frac{1}{C} \Mc^\T \Sb^\dagger \bmu_G.
\end{align*}
\eqref{eq:SB_def} implies that $\Sb= \frac{1}{C} \Mc \Mc^\T$. Thus,

\begin{align*}
    \W = & \Mc^\T \left(\Mc \Mc^\T\right)^\dagger = \Mc^\dagger \nonumber  \\
    \b = & \frac{1}{C} \one_C - \Mc^\T \left(\Mc \Mc^\T\right)^\dagger \bmu_G
    = \frac{1}{C} \one_C - \Mc^\dagger \bmu_G.
\end{align*}
\ANC{2} implies that $\Mc$ has exactly $C-1$ non-zero and \textit{equal} singular values, so $\Mc^\dagger = \alpha \Mc^\T$ for some constant $\alpha$. Combining the previous pair of displays, we obtain

\begin{subequations}
\begin{align}
     \W = & \alpha \Mc^\T \label{eq:self-duality} \\ 
     \b = & \frac{1}{C} \one_C - \alpha \Mc^\T \bmu_G;
\end{align}
\end{subequations}
\eqref{eq:self-duality} demonstrates the asserted self-duality \ANC{3}, up to rescaling. The class predicted by the above classifier is given by:

\begin{align*}
    & \argmax_{c'} \left< \w_{c'}, \h \right> + b_{c'} \nonumber \\
    = & \argmax_{c'} \alpha \left< \bmu_{c'} - \bmu_G, \h \right> + \frac{1}{C} - \alpha \left< \bmu_{c'} - \bmu_G, \bmu_G \right> \nonumber \\
    = & \argmax_{c'} \left< \bmu_{c'} - \bmu_G, \h - \bmu_G \right>.
\end{align*}
Using the equal norm property of \ANC{2}, this display becomes

\begin{align}
    & \argmax_{c'} \left< \bmu_{c'} - \bmu_G, \h - \bmu_G \right> \nonumber \\
    = & \argmin_{c'} \|\h - \bmu_G\|_2^2 - 2 \left< \bmu_{c'} - \bmu_G, \h - \bmu_G \right> + \|\bmu_{c'} - \bmu_G\|_2^2 \nonumber \\
    = & \argmin_{c'} \|(\h - \bmu_G) - (\bmu_{c'} - \bmu_G)\|_2 \nonumber \\
    = & \argmin_{c'} \|\h - \bmu_{c'}\|_2.\label{eq:nn_end}
\end{align}
In words, the decision of the linear classifier based on $(\W,\b)$ is identical to that made by NCC \ANC{4}.
\end{proof}

Our theorem predicts that evidence of (NC1)-(NC2) as shown in Figures \ref{fig:equinorm_std}, \ref{fig:equi_angle}, \ref{fig:max_angle}, and \ref{fig:collapse}
should deterministically accompany both (NC3) and (NC4) -- exactly as observed in our Figures \ref{fig:w_minus_h} and \ref{fig:nn_match}.

\subsection{Soudry et. al. (2018)}

The authors of \cite{soudry2018implicit} consider $C$-class classification, again in a setting where the parameter vector $\btheta$ is not trained, so that the last-layer activations $\h_{i,c} = \h(\x_{i,c})$  are fixed and not subject to feature engineering. 

They proved an important result which explicitly addresses our paper's focus on cross-entropy classifier loss minimization in the zero-error regime.

\begin{proposition}[Theorem 7 in \cite{soudry2018implicit}] \label{prop:nati}
    Let  $\R^{NC \times p}$ denote the vector space spanning all last-layer activation datasets, $\H = (\h_{i,c}: 1 \leq i \leq N , 1 \leq c \leq C)$, and let $\cH$ denote the measurable subset of $\R^{NC \times p}$ consisting of linearly separable  datasets, i.e. consisting of datasets $\H$ where, for {\it some} linear classifiers $\{\w_c\}_{c=1}^C$ (possibly depending on dataset $\H$), separability holds:
    \[
     \langle \w_c - \w_c' , \h_{i,c}\rangle \geq 1,  \qquad \h_{i,c}  \in \H.
    \]
    For (Lebesgue-) almost every dataset $\H \in \cH$, gradient descent minimizing the cross-entropy loss, as a function of the classifier weights, tends to a limit. This limit is identical to the solution of the \textit{max-margin classifier problem}:
    \begin{align}\label{eq:nati}
        \min_{\{\w_c\}_{c=1}^C} \sum_{c=1}^C \|\w_c\|_2^2
        \ \text{s.t.} \
        \forall i, c, c' \neq c:
        \left< \w_c - \w_{c'}, \h_{i,c} \right> \geq 1.
    \end{align}
\end{proposition}
This inspiring result significantly constrains the form of the trained classifier, in precisely the cross-entropy loss setting relevant to deep learning. However, because feature activations are here fixed, not learned, this result is only able to give indirect, implicit information about a deepnet trained model involving feature engineering and about the classification decisions that it makes.

Some authors of \cite{soudry2018implicit} have, in a series of highly influential talks and papers, laid great emphasis on the notion of an emergent, not superficially evident, \textbf{\textit{`inductive bias'}} as a reason for the surprising success of deep learning training, and have pointed to Proposition \ref{prop:nati} as a foundational result indicating that  `inductive bias' can be implicit in the behavior of a training procedure that superficially shows no behavioral tendencies in the indicated direction.

We agree wholeheartedly with the philosophy underlying \cite{soudry2018implicit}; our results support the further observation that {\it inductive bias is far more constraining on the outcome of modern deepnet training than was previously known.}

In effect, out of all possible max-margin classifiers that could be consistent with Proposition \ref{prop:nati}, the modern deepnet training paradigm is producing linear classifiers approximately belonging to the very tiny subset with the additional property of being Simplex ETFs. Moreover, such classifiers exhibit very striking behavioral simplicity in decision making.

\begin{theorem} [\textbf{Prop. \ref{prop:nati}+\ANC{1-2} imply \ANC{3-4}}]
Adopt the framework and assumptions of Proposition \ref{prop:nati}, as well as the end state implied by (NC1)-(NC2), i.e. \ANC{1}-\ANC{2}. The Soudry et. al. classifier \eqref{eq:nati}, in this setting, has the additional properties \ANC{3}-\ANC{4}.
\end{theorem}

\begin{proof}
Since $\Mc$ is the matrix of a Simplex ETF, it has $C-1$ equal-sized singular values with the remaining singular value being zero. Without loss of generality, we assume here that those singular values are 1, i.e. $\| \Mc \|_2 = 1$. Notice the singular value assumption implies the columns are of norm $\| \bmu_c - \bmu_G \|_2 = \sqrt{(C-1)/C}$ (\textit{not} unity) for all $c$.
At the assumed end-state of variability collapse \ANC{1}, the activations all collapse to their respective class-means, and the max-margin classifier problem reduces to

\begin{equation*}
    \min_{\{\w_c\}_{c=1}^C} \sum_{c=1}^C \frac{1}{2} \|\w_c\|_2^2
    \ \ \text{s.t.} \ \
    \forall c, c' \neq c:
    \left< \w_c - \w_{c'}, \bmu_c - \bmu_G \right> \geq 1.
\end{equation*}
Rewriting with matrix notation and using a Pythagorean decomposition of the objective, the above becomes:

\begin{align*}
    \min_{\W} \frac{1}{2} \|\W \Mc \Mc^\dagger \|_F^2 + \frac{1}{2} \|\W (\I - \Mc \Mc^\dagger)\|_F^2
    \\
    \text{s.t.} \ \
    \forall c, c' \neq c:
    (\e_c - \e_{c'})^\T \W \Mc \e_c \geq 1,
\end{align*}
where $\dagger$ denotes the Moore–Penrose pseudoinverse of a matrix. Property \ANC{2} fully specifies the Gram matrix of $\Mc$, from which we know that $\Mc$ has $C-1$ singular values all equal to one, WLOG, and a right-nullspace spanned by the vector of ones, $\one$, since its columns have zero mean. Thus, its singular value decomposition is given by $\Mc = \U \V^\T$, where $\U \in \R^{p \times C-1}$ and $\V \in \R^{C \times C-1}$ are partial orthogonal matrices. Hence,

\begin{align*}
    \min_{\W} \frac{1}{2} \|\W \U \U^\T \|_F^2 + \frac{1}{2} \|\W (\I - \U \U^\T)\|_F^2
    \\
    \text{s.t.} \ \
    \forall c, c' \neq c:
    (\e_c - \e_{c'})^\T \W \U \V^\T \e_c \geq 1.
\end{align*}
Observe that: (i) the second term of the objective penalizes deviations of $\w_c$ from the columnspace of $\U$; and (ii) such deviations do not affect the constraints. Conclude that the optimal solution for the above optimization problem has the form $\W = \A \U^\T$, where $\A \in \R^{C \times C-1}$. This simplification, as well as the fact that

\begin{align*}
    \|\W \U \U^\T\|_F^2
    = & \|\A \U^\T \U \U^\T\|_F^2 \\
    = & \Tr{\A \U^\T \U \U^\T \U \U^\T \U \A}
    = \|\A\|_F^2
\end{align*}
and $\W \U = \A \U^\T \U = \A$, transforms the optimization problem into the equivalent form

\begin{align}\label{eq:opt_original}
    \min_{\A} \frac{1}{2} \|\A \|_F^2
    \ \ \text{s.t.} \ \
    \forall c, c' \neq c:
    (\e_c - \e_{c'})^\T \A \V^\T \e_c \geq 1.
\end{align}
Averaging the constraints of \eqref{eq:opt_original} over $c$ and summing over $c' \neq c$, we obtain

\begin{align*}
    C-1
    \leq &
    \frac{1}{C} \sum_c \left( (C-1) \e_c - (\one - \e_c) \right)^\T \A \V^\T \e_c \\ 
    = & \frac{1}{C} \sum_c \left( C \e_c - \one \right)^\T \A \V^\T \e_c \\
    = & \sum_c \e_c^\T \left( \I - \frac{1}{C} \one \one^\T \right) \A \V^\T \e_c \\
    = & \Tr{\left( \I - \frac{1}{C} \one \one^\T \right) \A \V^\T} \\
    = & \Tr{\A \V^\T \left( \I - \frac{1}{C} \one \one^\T \right)} \\
    = & \Tr{\A \V^\T},
\end{align*}
where the last equality follows from $\V^\T \ \one = 0$. This leads to the following relaxation of \eqref{eq:opt_original}:

\begin{gather} \label{eq:opt_new}
    \min_{\A} \frac{1}{2} \|\A \|_F^2
    \ \ \text{s.t.} \ \
    \Tr{\A \V^\T} \geq C-1.
\end{gather}
Checking first-order conditions, the optimum occurs at $\A = \V$.
Recalling that $\Mc$ is a Simplex ETF with singular values 1, $\V \V^\T = \V \U^\T \U \V^\T = \Mc^\T \Mc = \I - \frac{1}{C} \one \one^\T$. Because $\A=\V$, and $\V \e_c = \left(\e_c - \frac{1}{C}\one\right)$  for $c=1,\dots,C$,
\begin{equation}
(\e_c - \e_{c'})^\T \A \V^\T \e_c = (\e_c - \e_{c'})^\T \V \V^\T \e_c = 1.
\end{equation}
Since $\A=\V$ optimizes \eqref{eq:opt_new}, which involves the same objective as \eqref{eq:opt_original}, but over a possibly enlarged feasible set, feasibility of $\A=\V$ implies that $\A=\V$ optimizes \eqref{eq:opt_original} as well. The solution to \eqref{eq:opt_original} is unique, since the problem minimizes a positive definite quadratic subject to a single nondegenerate linear constraint. In the optimization problem for $\W$ that we started with, recall that $\W=\A\U^\T$. Hence, the optimality of $\A=\V$ implies $\W = \A \U^\T = \V \U^\T = \Mc^\T$, showing self-duality is achieved \ANC{3}. This equality becomes a proportionality in the more general case where the equal singular values of $\Mc$ are not unity.


An argument similar to the one for Theorem \ref{thm:collapse_web_low} that the classifier is behaviorally equivalent to the NCC decision rule \ANC{4}.
\end{proof}

Much like Theorem \ref{thm:collapse_web_low}, but now for cross-entropy loss, the above result again indicates that evidence of (NC1)-(NC2) as shown in Figures \ref{fig:equinorm_std}, \ref{fig:equi_angle}, \ref{fig:max_angle}, and \ref{fig:collapse}
should accompany both (NC3) and (NC4), as shown in Figures \ref{fig:w_minus_h} and \ref{fig:nn_match}. In short, our results indicate an inductive bias towards NCC which is \textit{far more total and limiting} than the max-margin bias proposed by \cite{soudry2018implicit}.

\section{Theoretical derivation of Simplex ETF emergence}
We are unaware of suggestions, prior to this work, that Simplex ETFs emerge as the solution of an interesting and relevant optimization problem. Prompted by the seemingly surprising nature of the above empirical results, we developed theoretical results which show that the observed end-state of Neural Collapse can be derived directly using standard ideas from information theory and probability theory. Roughly speaking, the Simplex ETF property \ANC{2}, self-duality \ANC{3}, and behavioral simplification \ANC{4} are derivable  consequences of variability collapse \ANC{1}.

In our derivation, we consider an abstraction of feature engineering, in which an ideal feature designer chooses activations which minimize the classification error in the presence of \textit{nearly-vanishing} within-class variability. Our derivation shows that the ideal feature designer should choose activations whose class means form a Simplex ETF. 

\subsection{Model assumptions}\label{sec:infotheory_model}
Assume we are given an observation $\h = \bmu_\gamma + \z \in \R^C$, where $\z\sim \mathcal{N}(\zr,\sigma^2\I)$ and $\gamma \sim \text{Unif} \{ 1,\dots, C\}$ is an {\it unknown} class index, distributed independently from $\z$. Our goal is to recover $\gamma$ from $\h$, with as small an error rate as possible. We constrain ourselves to use a linear classifier, $\W\h + \b$, with weights $\W = [\w_c: c=1,\dots,c] \in \R^{C \times C}$ and biases $\b = (b_c) \in \R^C$; Our decision rule is
\[
    \hat{\gamma}(\h) =  \hat{\gamma}(\h; \W,\b) = \arg\max_c  \langle \w_c, \h \rangle + b_c.
\]
Our task is to design the classifier $\W$ and bias $\b$, as well as a matrix $ \M = [\bmu_c: c=1,\dots,C] \in \R^{C\times C}$, subject to the norm constraints $\| \bmu_c\|_2 \leq 1$ for all $c$.

\subsection{Information theory perspective}\label{sec:infotheory_infoperspective}
The above can be recast as an optimal coding problem in the spirit of Shannon \cite{shannon1959probability}. The class means $\bmu_c$ are \textit{codewords} and the matrix $\M$ represents a \textit{codebook}, containing $C$ codewords. A transmitter transmits a codeword over a noisy channel, contaminated by white additive Gaussian noise, and then a receiver obtains the noisy signal $\h = \bmu_c + \z$ which it then {\it decodes} using a linear decoder $\hat{\gamma}$ in an attempt to recover the transmitted $\gamma$. The norm constraint on the means captures limits imposed on signal strength due to the distance between the transmitter and receiver. Our task is to design a codebook and decoder that would allow optimal retrieval of the class identity $\gamma$ from the noisy information $\h$. 

\subsection{Large-deviations perspective}\label{sec:infotheory_probperspective}
To measure success in this task, we consider the
large-deviations error exponent:
\[
\beta(\M,\W,\b) =  - \lim_{\sigma \to 0} \sigma^{2} \log P_\sigma \{ \hat{\gamma}(\h)  \neq \gamma \} .
\]
{\it This is the right limit, as we are considering the situation where the noise is approaching zero due to variability collapse (NC1)}. Tools for deriving large deviations error exponents have been extensively developed in probability theory \cite{dembo1998large}.

\subsection{Theoretical result}
As a preliminary reduction, we can assume without loss of generality that
the ambient vector space, in which the codewords and observations lie, is simply $\R^C$ (see SI Appendix).
\begin{theorem}
Under the model assumptions just given in subsections \ref{sec:infotheory_model}, \ref{sec:infotheory_infoperspective}, and \ref{sec:infotheory_probperspective}, the Optimal Error Exponent is

\begin{align*}
  \beta^\star = & \max_{\M,\W,\b}   \beta(\M,\W,\b) \ \ \text{s.t.} \ \ \|\bmu_c\|_2 \leq 1 \ \forall c \\
  = &  {\frac{C}{C-1}} \cdot \frac{1}{4},
\end{align*}
where the maximum is over $C \times C$ matrices $\M$ with at most unit-norm columns, and over $C \times C$ matrices $\W$ and $C \times 1$ vectors $\b$.

Moreover, denote  $\M^\star = \sqrt{\frac{C}{C-1}} \left( \I - \frac{1}{C} \one\one^\T \right)$, i.e, $\M^\star$ is the standard Simplex ETF. The Optimal Error Exponent is precisely achieved by $\M^\star$:
\[
   \beta(\M^\star,\M^\star,\zr) = \beta^\star.
\]
All matrices $\M$ achieving $\beta^\star$ are also Simplex ETFs -- possibly in an isometric pose -- deriving from $\M^\star$ via $\M = \U \M^\star$ with $\U$ a $C \times C$ orthogonal matrix. For such matrices, 
an optimal linear decoder is $\W=\Mstar\U^\T$, $\b=\zr$:
\[
   \beta(\M,\W,\b) = \beta(\M,\M^\T,\zr) = \beta^\star.
\]
\end{theorem}
\begin{proof}
Given in SI Appendix.
\end{proof}

In words, if we engineer a collection of codewords to optimize the (vanishingly small)
theoretical misclassification probability, we obtain as our solution the standard Simplex ETF, or a rotation of it.

We stress that the {\it maximal equiangularity property of $\M^\star$ is crucial to this result},
i.e. 
\[
    \langle \bmu^\star_c, \bmu^\star_{c'} \rangle = \frac{-1}{C-1}, \qquad c' \neq c;
\]
this property is enjoyed by every collection of 
class-means optimizing the error exponent and is unique to Simplex ETFs.

The results of this section show that Simplex ETFs are the unique solution to an abstract optimal feature design problem. The fact that modern deepnet training practice has found this same Simplex ETF solution suggests to us that the training paradigm -- SGD, TPT and so on -- is finding the same solution as would an ideal feature engineer! Future research should seek to understand the ability of training dynamics to succeed in obtaining this solution.

\section{Related works}\label{sec:related_works}

The prevalence of Neural Collapse makes us view a number of previous empirical and theoretical observations in a new light.

\subsection{Theoretical feature engineering}
Immediately prior to the modern era of purely empirical deep learning, \cite{bruna2013invariant} proposed a theory-derived machinery building on the scattering transform that promised an understandable approach for handwritten digit recognition. The theory was particularly natural for problems involving within-class variability caused by `small' morphings of class-specific templates; In fact, the scattering transform was shown  in \cite{mallat2012group} to tightly limit the  variability caused by template morphings. Later, \cite{wiatowski2017mathematical,wiatowski2016discrete,wiatowski2015deep}, complemented \cite{bruna2013invariant} with additional theory covering a larger range of mathematically-derived features, nonlinearities, and pooling operations -- again designed to suppress within-class variability.

Our finding of Neural Collapse, specifically (NC1), shows that feature engineering by standard empirical deepnet training achieves similar suppression of within-class variability--both on the original dataset considered by \cite{bruna2013invariant} as well as six more challenging benchmarks. Thus, the original goal of \cite{mallat2012group,bruna2013invariant,wiatowski2017mathematical,wiatowski2016discrete,wiatowski2015deep}, which can be phrased as the limiting of within-class variability of activations, turns out to be possible for a range of datasets; and, perhaps more surprisingly, to be learnable by stochastic gradient descent on  cross-entropy loss. Recently, Mallat and collaborators were able to deliver results with scattering-transform features (combined with dictionary learning) that rival the foundational empirical results produced by AlexNet \cite{zarka2019deep}. So apparently, controlling within-class activation variability, whether this is achieved analytically or empirical, is quite powerful.

\subsection{Observed structure of spectral Hessians}

More recently, empirical studies of the Hessian of the deepnet training loss of image-classification networks
observed surprising and initially baffling deterministic structure. First observed by \cite{sagun2016eigenvalues,sagun2017empirical}, on toy models, the spectrum exhibits $C$ outlier eigenvalues separated from a bulk, where $C$ is the number of classes of the image classification task.
\cite{papyan2018full,papyan2019measurements,ghorbani2019investigation} corroborated these findings at scale on modern deep networks and large datasets. \cite{papyan2018full,papyan2019measurements} explained how the spectral outliers could be attributed to low-rank structure associated with class-means and the bulk could be induced by within-class variations (of logit-derivatives). It was essential that the class means have greater norm than the within-class standard deviation in order for these spectral outliers to emerge.

Under (NC1), the full matrix of last-layer activations converges to a rank-$(C{-}1)$ matrix, associated with class-means. So under (NC1), eventually the within-class standard deviation will be much smaller, and the outliers will emerge from the bulk. In short, the collapse of activation variability (NC1), combined with convergence of class means (NC2) to the Simplex ETF limit, explains these important and highly visible observations about deepnet Hessians.

\subsection{Stability against random and adversarial noise}
It is well understood classically that when solving linear systems $\y = \M \x$ by standard methods, some matrices $\M$ are prone to solution instability, blowing up small noise in $\y$ to produce large noise in $\x$; other matrices are less prone. Stability problems arise if the nonzero singular values of $\M$ are vastly different and don't arise if the nonzero singular values are all identical. The Simplex ETF offers equal nonzero singular values, and so a certain resistance to noise amplification. This is a less well known path to equal singular values, partial orthogonal matrices being of course the more well known.

In the deepnet literature, the authors of \cite{papyan2017convolutional,romano2019adversarial,sulam2019multi,aberdam2019multi,aberdam2020and} studied the stability of deepnets to adversarial examples. They proposed that stability can be obtained by making the matrices defined by the network weights close to orthogonal. However, no suggestion was offered for why trained weights, under the current standard training paradigm, would tend to become orthogonal.

In \cite{cisse2017parseval}, the authors modified the standard training paradigm, forcing linear and convolutional layers to be approximate tight frames; they showed this leads both to better robustness to adversarial examples, as well as improved accuracy and faster training. To get these benefits, they {\it imposed} orthogonality explicitly during  training.

Both \cite{papyan2017convolutional,romano2019adversarial} and \cite{cisse2017parseval} showed how concerns about stability can be addressed by explicit interventions in the standard training paradigm. By demonstrating  a pervasive Simplex ETF structure, this paper has shown that, under today's standard training paradigm, deepnets naturally achieve an \textit{implicit} form of stability in the last-layer. In light of the previous discussions of the benefits of equal singular values, we of course expected the trained deep network would become more robust to adversaries, as the training progresses towards the Simplex ETF. The measurements we reported here support this prediction, and evidence in \cite{deniz2020robustness} gives further credence to this hypothesis.

\section{Conclusion}
This paper studied the terminal phase of training (TPT) of today's canonical deepnet training protocol. It documented that during TPT a process called Neural Collapse takes place, involving four fundamental and interconnected phenomena: (NC1)-(NC4).

Prior to this work, it was becoming apparent, due to \cite{soudry2018implicit} and related work, that the last-layer classifier of a trained deepnet exhibits appreciable mathematical structure -- a phenomenon called `inductive bias' which was gaining ever-wider visibility. Our work exposes considerable additional fundamental, and we think, surprising, structure: (i) the last-layer features are not only linearly separable, but actually collapsed to a $C$-dimensional Simplex ETF, and (ii) the last-layer classifier is behaviorally equivalent to the Nearest Class-Center decision rule. Through our thorough experimentation on seven canonical datasets and three prototypical networks, we show that these phenomena persist across the range of canonical deepnet classification problems. Furthermore, we document that convergence to this simple structure aids in the improvement of out-of-sample network performance and robustness to adversarial examples. We hypothesize that the benefits of the interpolatory regime of overparametrized networks are directly related to Neural Collapse.

From a broader perspective, the standard workflow of empirical deep learning can be viewed as a series of arbitrary steps that happened to help win prediction challenge contests, which were then proliferated by their popularity among contest practitioners. Careful analysis, providing a full understanding of the effects and benefits of each workflow component, was never the point. One of the standard workflow practices is training beyond zero-error to zero-loss, i.e. TPT. In this new work, we give a clear understanding that TPT benefits today's standard deep learning training paradigm by showing how it leads to the pervasive phenomenon of Neural Collapse. Moreover, this work puts older results on a new footing, expanding our understanding of their contributions. Finally, because of the precise mathematics and geometry, the doors are open for new formal insights.  

\acknow{This work was partially supported by NSF DMS 1407813, 1418362, and 1811614 and by private donors. Some of the computing for this project was performed on the Sherlock cluster at Stanford University; we thank the Stanford Research Computing Center for providing computational resources and support that enabled our research. Some of this project was also performed on Google Cloud Platform: thanks to Google Cloud Platform Education Grants Program for research credits that supplemented this work. Moreover, we thank Riccardo Murri and Hatef Monajemi for their extensive help with the Elasticluster and ClusterJob frameworks, respectively.}

\showacknow{} 

\bibliography{nc_bib}

\pagebreak
\appendix
\part*{Supplementary Material}
\newtheorem{innercustomthm}{Theorem}
\newenvironment{customthm}[1]
  {\renewcommand\theinnercustomthm{#1}\innercustomthm}
  {\endinnercustomthm}

\newtheorem{thm}{Theorem}
\newtheorem{lem}[thm]{Lemma}
\newtheorem{crl}[thm]{Corollary}
\newtheorem{prop}[thm]{Proposition}
\newtheorem{Definition}{Definition}

\newcommand{\Vxh}{\hat{\Vx}}
\newcommand{\Vy}{\ensuremath{{b}}}
\newcommand{\Vxt}{\ensuremath{{\tilde{\Vx}}}}
\newcommand{\Mx}{\ensuremath{{X}}}
\newcommand{\My}{\ensuremath{{Y}}}
\newcommand{\Mxt}{\ensuremath{\tilde{X}}}
\newcommand{\Mxh}{\ensuremath{\hat{X}}}
\newcommand{\p}{\boldsymbol{p}}
\newcommand{\q}{\boldsymbol{q}}

\newcommand{\Sx}{\ensuremath{x}}
\newcommand{\Shy}{\ensuremath{y}}
\newcommand{\Sxh}{\ensuremath{\hat{x}}}
\newcommand{\bitem}{\begin{itemize}}
\newcommand{\eitem}{\end{itemize}}

\newcommand{\D}{\boldsymbol{D}}
\renewcommand{\E}{\mathcal{E}}
\newcommand{\bdelta}{\boldsymbol{\delta}}
\newcommand{\K}{\mathcal{K}}
\newcommand{\bu}{\boldsymbol{u}}
\renewcommand{\S}{\boldsymbol{S}}
\newcommand{\cS}{\mathcal{S}}
\newcommand{\bnu}{\boldsymbol{\nu}}
\newcommand{\F}{\mathcal{F}}
\renewcommand{\i}{\boldsymbol{i}}
\renewcommand{\r}{{\boldsymbol{r}}}
\renewcommand{\L}{\mathcal{L}}
\newcommand{\bLambda}{\boldsymbol{\Lambda}}
\renewcommand{\H}{\boldsymbol{H}}
\newcommand{\G}{\boldsymbol{G}}
\renewcommand{\v}{\boldsymbol{v}}
\newcommand{\bV}{\boldsymbol{V}}
\newcommand{\Z}{\boldsymbol{Z}}
\newcommand{\B}{\mathcal{B}}
\renewcommand{\b}{\boldsymbol{b}}

\section{Setup}
\newcommand{\bh}{\boldsymbol{h}}
\newcommand{\goto}{\rightarrow}
Suppose we `feature engineer' (i.e., in some way, design) 
a matrix $\M$ of feature activation class means, 
with columns $[\bmu_c: c=1,\dots,C]$. We are given an observation $\bh = \bmu_\gamma + \z$, 
$\z \sim \mathcal{N}(\zr,\sigma^2 \I)$, where $\gamma$ is an {\it unknown} class index, 
$\gamma \in \{1,\dots,C\}$. Moreover, we assume that 
$\gamma \sim \mbox{unif} \{ 1,\dots, C\}$ independently from $\z$.
Our task is to recover $\gamma$ from $\h$, with as small an error rate
as possible. Our basic question is
\begin{quotation}
\sl Which feature means matrices $\M$ will enable the optimal error rate?
\end{quotation}

In Information Theory terminology,
the feature means $\bmu_c$ are {\it codewords},
and the matrix $\M$ is a {\it codebook} containing $C$ {\it codewords}. A transmitter transmits a codeword over a noisy channel, contaminated by white additive Gaussian noise, and then a receiver obtains the noisy signal $\h = \bmu_\gamma + \z$ which it then {\it decodes} in an attempt to recover the transmitted $\gamma$. Our task is to design a codebook and decoder that would allow optimal retrieval of the class identity $\gamma$ from the noisy information $\h$. Using the language of Information Theory \cite{shannon1959probability},
we could speak of {\it codebook design},
rather than {\it feature engineering} from Machine Learning. 
We will use a {\it linear decoder}, 
with weights $\W = [\w_c: c=1,\dots,C]$ and biases $\b = (b_c)$:
\[
        \hat{\gamma}(\bh) =  \hat{\gamma}(\bh; \W,\b) \equiv \mbox{argmax}_c  \langle \w_c, \h \rangle + b_c
\]
In this language, our question then becomes
\begin{quotation}
\sl Which codebook $\M$ and linear decoder $\W,\b$ will enable the optimal error rate?
\end{quotation}

We mention a preliminary reduction: we assume without loss of generality that
the ambient vector space $\V$, say, in which the codewords and observations lie, is simply $\R^C$.
Indeed, if $\V$ were larger it could not possibly help. The linear span $\mbox{lin}(\{\bmu_c\})$ is at most $C$-dimensional.
The orthocomplement of the linear span $\mbox{lin}(\{\bmu_c\})$ is useless; the observations $\h$
projected onto such an orthocomplement would simply be standard Gaussian noise
with a distribution that is invariant to the choice of $\gamma$. 
Applying a sufficiency
argument from statistical decision theory -- see \cite{lehmann2006testing} --
completes the reduction to $\V = \mbox{lin}(\{\bmu_c\})$.
In effect,
any performance we can get with a larger space $\V$ is also
available to us with the reduction to $\mbox{lin}(\{\bmu_c\})$;
we do not need the orthocomplement  $\mbox{lin}(\{\bmu_c\})^\perp$
as an additional random noise generator.

In addition, there is no benefit either for adopting an ambient $C$-dimensional vector space
different than $\V = \R^C$, eg. one which depends on $\M$. The decision problem itself is invariant under orthogonal transformations;
namely, if we replace any tuple $(\M,\W,\b)$  by $(\U\M,\W \U^\T,\b)$
where $\U$ is an orthogonal transformation, we get the identical performance,
since the Gaussian noise distribution
is invariant to orthogonal transformations. Therefore, any performance we might get with 
an idiosyncratic $C$-dimensional realization of $\V$, we can get with
the canonical realization space $\V \equiv \R^C$.

\section{Theorem 5 from main manuscript}
To measure success in this task, we consider the
{\it Large-Deviations Error Exponent}:
\[
\beta(\M,\W, \b) =  \lim_{\sigma \goto 0} -\sigma^{2} \log P_\sigma \{ \hat{\gamma}(\h)  \neq \gamma \} 
\]

\begin{customthm}{5}
The Optimal Error Exponent is
\begin{align*}
  \beta^\star = & \max_{\M,\W,\b}   \beta(\M,\W,\b)  \ \ \text{s.t.} \ \ \|\bmu_c\|_2 \leq 1 \ \forall c \\
  = & {\frac{C}{C-1}} \cdot \frac{1}{4},
\end{align*}
where the maximum is over $C \times C$ matrices $\M$ with at most unit-norm columns, $C \times C$ matrices $\W$, and $C \times 1$ vectors $\b$.

Moreover, denote  $\M^\star = \sqrt{\frac{C}{C-1}} \left( \I - \frac{1}{C} \one\one^\T \right)$, i.e, $\M^\star$ has zero mean columns and is the standard Simplex ETF. The Optimal Error Exponent is precisely achieved by $\M^\star$:
\[
   \beta(\M^\star,\M^\star,\zr) = \beta^\star.
\]
All matrices $\M$ achieving $\beta^\star$ are also Simplex ETFs -- possibly in another pose -- deriving from $\M^\star$ via $\M = \U \M^\star$ with $\U$ a $C \times C$ orthogonal matrix. For such matrices, 
an optimal linear decoder is $\W=\Mstar\U^\T$, $\b=\zr$:
\[
   \beta(\M,\W,\b) = \beta(\U\Mstar,\Mstar\U^\T,\zr) = \beta^\star.
\]
\end{customthm}
\begin{proof}
The proof follows from a series of lemmas -- established in the following pages -- and is given in Section \ref{sec:beta_optimality}.\ref{subsec:proof_thm5}.
\end{proof}

\section{Large Deviations}
\subsection{Basic large deviations, Gaussian White Noise}

\begin{lem}\label{lem:large-dev}
Suppose that $ \zr \not \in \K$, and that $\K$ is a closed set.
Suppose that $\z \sim \mathcal{N}(\zr,\sigma^2 \I)$. Then, as $\sigma \goto 0$:
\[
         -  \sigma^{2} \log P_\sigma \{  \z \in \K \}  \goto \min\left\{ \frac{1}{2} \| \z \|_2^2 : \z \in \K \right\}.
\]
\end{lem}
\begin{proof}
See \cite{dembo1998large} and results therein.
\end{proof}
This lemma defines an optimization problem:
\[
 (P_{LD})  \qquad \min\left\{ \frac{1}{2} \| \z \|_2^2 : \z \in \K \right\}.
\]
Denote the solution of  the optimization problem
$(P_{LD})$ by $\z^\star(\K)$ and the value of the  optimization problem
by $\beta(\K) = \frac{1}{2} \| \z^\star(\K) \|_2^2$.
The solution $\z^\star(\K)$ is the closest point in $\K$ to $\{\zr\}$.
Conceptually, $\z^\star(\K)$ is the ``most likely way'' for the ``rare event'' $ \z \in \K$ to happen.
The likelihood of this rare event obeys
\[
-  \log P_\sigma \{  \z \in \K \}   \sim \sigma^{-2} \cdot  \beta(\K) , \qquad \sigma \goto 0.
\]

\subsection{Fundamental events causing misclassification}

In this section, we identify fundamental events causing misclassification, and apply the large deviations result from Lemma \ref{lem:large-dev} to study the misclassification probability $P \{ \hat{\gamma}(\bh) \neq \gamma \}$.

Consider the event: $\E_{c,c'} =$ ``item from true underlying class $c$ is misclassified as $c'$.''
Correspondingly, consider the larger event 
\[
\F_{c,c'} = \{ \mbox{Linear classifier score for $c'$ is at least as large as for $c$} \}.
\]
While $\F_{c,c'}$ does not by itself imply $\E_{c,c'}$, of course 
\[
\E_{c,c'} = \F_{c,c'} \cap  \left(\cap_{c'' \not \in \{ c,c'\}} \F_{{c'',c'}}\right).
\]
Moreover, consider the event:
$\E_{c} =$ ``item from true underlying class $c$ is misclassified.'' Then,
\[
\E_{c}  = \cup_{c' \neq c} \F_{c,c'}.
\]
So the events $\F_{c,c'}$ are fundamental.

Let  $\V = \R^C$ denote our ambient vector space.
Define the cone $\V_{c,c'} = \{ \v \in \V : v(c') \geq v(c) \}$.
Then
 $\F_{c,c'} = \{ \h: \W \h + \b \in \V_{c,c'} \}$.
Applying large deviations analysis as $\sigma \goto 0$:
\begin{equation}\label{eq:z(K)}
   - \sigma^{2} \log P_\sigma (\F_{c,c'}) \goto \min \left\{ \frac{1}{2} \| \z \|_2^2 : \z \in \K_{c,c'} \right\},
\end{equation}
where
\newcommand{\cW}{{\cal W}}
\begin{eqnarray*}
   \K_{c,c'} &=& \{ \z: \W (\bmu_c + \z) + \b \in \V_{c,c'} \} . 
\end{eqnarray*}
Conceptually, $\z^\star(\K_{c,c'})$, the optimal solution to \eqref{eq:z(K)}, is the most likely way noise can cause a
`pre-misclassification' of $c$ as $c'$.
Label $\z^\star_{c,c'} = \z^\star(\K_{c,c'})$; 
set $\beta_{c,c'} = \frac{1}{2} \|\z^\star_{c,c'}\|_2^2$.

Considering the misclassification event $\E_{c}$,
a large deviations analysis as $\sigma \goto 0$ gives:
\[
   - \sigma^{2} \log P_\sigma \{\E_{c}\} \goto   \min_{c' \neq c}  \beta_{c,c'}.
\]
Defining the LD exponent,
\[
          \beta_c = \min_{c' \neq c}  \beta_{c,c'},
\]
we have
\[
 - \log P_\sigma \{\E_{c}\} \sim \sigma^{-2} \beta_c, \qquad \sigma \goto 0.
 \]
 Finally, for the misclassification event $\E = \cup_{c} \E_{c}$,
we have the LD exponent,
\[
          \beta = \min_{ c}  \beta_c,
\]
for which we can say
\[
 - \log P_\sigma (\E) \sim \sigma^{-2} \beta, \qquad \sigma \goto 0.
\]
Thus, in this setting, minimizing the misclassification probability corresponds to maximizing $\beta$. This motivates the optimization problem studied in the following sections.

\section{Optimization Interpretation}

Consider the optimization problem with variable $\z = (\z_{c,c'}: c \neq c') \in R^{C \cdot (C(C-1))}$,
with each component $\z_{c,c'} \in \R^C$:
\begin{equation}
    (P_{\M,\W,\b})  \;\; \min_{\z} \min_{c' \neq c}  \frac{1}{2} \| \z_{c,c'} \|_2^2  \mbox{ subject to } \z_{c,c'} \in \K_{c,c'}.  \label{eq:PMAxDef}
\end{equation}
Denote the optimum as  $\z^\star = (\z^\star_{c,c'}: c \neq c')$ (in the cases of interest here it will be unique).
Although phrased as a multi-component optimization problem
across components $(\z_{c,c'})$, it is actually separable,
so $\beta_{c,c'} = \frac{1}{2} \| \z^\star_{c,c'} \|_2^2$. Moreover,
the value of the optimization problem, $\mbox{\sc val}(P_{\M,\W,\b})$,
is actually $\beta \equiv \min_{c' \neq c} \beta_{c,c'}$.

The value of the optimization problem,  $ \beta = \beta(\M,\W,\b) = \mbox{\sc val}(P_{\M,\W,\b})$, 
implicitly defines a function of $\M$, $\W$, and $\b$.  This notation shows that the LD exponent of
misclassification error depends on $\M$ the codebook and the linear classifier $(\W,\b)$.

Recall the main problem we are trying to solve in this supplement:
\begin{quotation}
\sl Which codebook $\M$ and linear decoder $\W,\b$ will enable the optimal error rate?
\end{quotation}
Using our new notation, this problem can be stated as follows: 

\begin{quotation}
\sl Which tuples $(\M,\W,\b) \in \R^{C^2} \times \R^{C^2} \times \R^C$ 
achieve the following optimum?
\begin{equation}\label{eq:BetaOptDef}
  \beta^{\star}_C = \max_{\M: \|\M\|_{2,\infty} \leq 1 } \sup_{\W \in L(\V,\V), \b \in \V} \beta(\M,\W,\b).
\end{equation}
\end{quotation}

\section{A Lower Bound}

Suppose we are given $\bmu_{c}$ and $\bmu_{c'}$,
that $\z \sim \mathcal{N}(\zr,\sigma^2 \I)$,
and that $\bh = \bmu_\gamma + \z$, where $\gamma \in \{ c,c'\}$.
Let $P_{c,\sigma}$ denote the probability measure governing $\bh$,
when $\gamma = c$ and $\sigma$ are as specified.
We have this fundamental  lower bound:
\begin{lem}\label{lem:neymann_pearson}
Consider the minimax test between $H_0: P_{c,\sigma}$ and $H_1: P_{c',\sigma}$,
minimizing the maximum of type I and type II errors. 
Let   $\delta = \frac{1}{2} \| \bmu_{c} - \bmu_{c'} \|_2$. 
The minimax error obeys:
\[
\max \big(P\{ \mbox{reject } H_0 | H_0 \},P\{ \mbox{accept } H_0 | H_1 \} \big)  
 =  P\{ \mathcal{N}(0,\sigma^2) > \delta \}  .
\]
\end{lem}

\begin{proof}
Consider the $1$-dimensional parametric family $Q_{\theta}$ for $\theta \in [0,1]$ with
$Q_0 =  P_{c,\sigma}$, $Q_1 =  P_{c',\sigma}$, where, in general, $Q_{\theta}$ is the probability measure
of $\bh = \bnu_\theta + \z$, and the mean vector $\bnu_\theta = \theta \bmu_{c} + (1-\theta) \bmu_{c'}$.
Note that $\delta = \| \bnu_{0} - \bnu_{1/2} \|_2$;  also, let $\bu = ( \bmu_c - \bmu_{c'}) / \|\bmu_c - \bmu_{c'} \|_2$, 
and note:
\[
\langle \bnu_0 -\bnu_{1/2}, \bu \rangle =   - \langle \bnu_1 - \bnu_{1/2}, \bu \rangle = \delta .
\]

In general, by sufficiency and standard factorization properties of the multivariate normal $\mathcal{N}(\zr, \sigma^2 \I)$
which governs $\z$, the Neyman-Pearson test between $H_0$ and $H_1$ has the form
\[
 \begin{array}{ll }
 \mbox{Accept $H_0$}:  & \langle \bh- \bnu_{1/2}, \bu \rangle >  0 \\
 \mbox{Reject $H_0$}: &\langle \bh-\bnu_{1/2}, \bu \rangle  \leq 0,
 \end{array}
\]
where the threshold $0$ derives from symmetry considerations.
When $H_1$ is true, 
\begin{eqnarray*}
 \langle \bh- \bnu_{1/2}, \bu \rangle &=&  \langle  \bnu_1 - \bnu_{1/2}, \bu \rangle +  \langle  \z, \bu \rangle \\
  &=_D& -\delta + \mathcal{N}(0,\sigma^2),
  \end{eqnarray*}
and
$P(\{ \mbox{accept } H_0 | H_1 \})   = P \{  \mathcal{N}(0, \sigma^2) > \delta \}$.
Similarly, when $H_0$ is true, 
\begin{eqnarray*}
 \langle \bh- \bnu_{1/2}, \bu \rangle &=&  \langle  \bnu_0 - \bnu_{1/2}, \bu \rangle +  \langle  \z, \bu \rangle \\
  &=_D& \delta + \mathcal{N}(0,\sigma^2),
  \end{eqnarray*}
and $P(\{ \mbox{reject } H_0 | H_0 \})   = P \{  \mathcal{N}(0, \sigma^2) < -\delta \} $.
\end{proof}

\begin{lem}
Let $\hat{\gamma}(\cdot)$ be a decision procedure that 
takes values in $\{c,c'\}$ and
suppose 
\[
    \E_{c,c'} = \{ \mbox{$\gamma = c$ but $\hat{\gamma} = c'$} \}.
\]
Suppose that $\z \sim \mathcal{N}(\zr,\sigma^2 \I)$, 
and that $\bh = \bmu_\gamma + \z$.
Let $P_{c,\sigma}$ denote the probability measure governing $\bh$,
when $\gamma = c$ and $\sigma$ are as specified.
Then, as $\sigma \goto 0$:
\[
         \liminf_{\sigma \goto 0}  - \sigma^{2}  \max_{ \gamma \in \{ c,c'\}  }\log P_{\gamma,\sigma} \{\E_{c,c'}\}  \leq \frac{1}{8} \| \bmu_c - \bmu_{c'}\|_2^2.
\]
\end{lem}

\begin{proof}
The rule $\hat{\gamma}$ cannot possibly have worst-case (across  $\gamma \in \{c,c'\}$) probability of error better than the
minimax test described in Lemma \ref{lem:neymann_pearson}. Hence,
\[
\max ( P_{c,\sigma} \{\E_{c,c'}\} , P_{c',\sigma} \{\E_{c',c}\}) \geq P\{ \mathcal{N}(0,\sigma^2) > \delta  \}.
\]
Now by Lemma \ref{lem:large-dev},
\[
 \lim_{\sigma \goto 0} -\sigma^{2}  \log P\{ \mathcal{N}(0,\sigma^2) > \delta  \} = \frac{1}{2} \delta^2.
\]
Since $\delta = \frac{1}{2} \| \bmu_{c} - \bmu_{c'} \|_2$, we obtain:
\[
         \liminf_{\sigma \goto 0}  - \sigma^{2}  \max_{ \gamma \in \{ c,c'\}  }\log P_{\gamma,\sigma} \{\E_{c,c'}\}  \leq \frac{1}{8} \| \bmu_c - \bmu_{c'}\|_2^2.
\]
\end{proof}

\begin{crl}
Let $\M$ be a given codebook matrix with columns $\{\bmu_c\}_{c=1}^C$. Define
\[
  \Delta(\M) = \min_{c' \neq c} \| \bmu_c - \bmu_{c'} \|_2.
\]
Then, for any decision rule $\hat{\gamma}$,
\[
         \liminf_{\sigma \goto 0}  - \sigma^{2}  \max_c \log P_{c,\sigma} \{ \E_c \}  \leq  \frac{1}{8} \Delta^2.
\]
In particular,
\begin{equation} \label{eq:betaDeltaIndivBnd}
         \beta(\M,\W,\b) \leq  \frac{1}{8} \Delta(\M)^2.
\end{equation}
\end{crl}

This motivates us to define the {\it maximin codeword distance},
\[
  \Delta^\star_C  \equiv  \max_{\M: \|\M\|_{2,\infty} \leq 1 }   \Delta(\M) = \max_{\M: \|\M\|_{2,\infty} \leq 1 } \min_{c' \neq c} \| \bmu_c - \bmu_{c'} \|_2 ,
\]
and the following question:
\begin{quotation}
\sl Which codebook matrices $\M \in \R^{C^2}$ achieve the maximin distance?
\end{quotation}
This distance controls the optimal $\beta$:
\begin{equation}
   \beta^{\star}_C \leq \frac{1}{8} (\Delta^\star_C)^2. \label{eq:BetaOptBound}
\end{equation}

\section{$\Delta$-Optimality of the Simplex Tight Frame}

\begin{lem} \label{lem:EvalDeltaSimplex}
Let $\M^\star =  \sqrt{\frac{C}{C-1}} \left(\I - \frac{1}{C} \one \one^\top \right)$.
Again with $\Delta$ being the minimum Euclidean distance between any two columns of $\Mstar$,
\[
     \Delta(\Mstar)  = \sqrt{\frac{2C}{C-1}}.
\]
\end{lem}

\begin{proof}
Let the columns of $\Mstar$ be denoted $\bmu_c^\star$ and those of $\I$ be denoted $\bdelta_c$, $c=1,\dots,C$.
By a side calculation $\|\bmu_c^\star\| =  1$, $c=1,\dots,C$. The result then follows from
\begin{align*}
\| \bmu_c^\star - \bmu_{c'}^\star\|_2 = & \sqrt{\frac{C}{C-1}} \left\| \left(\bdelta_c - \frac{1}{C} \one\right) -  \left(\bdelta_{c'} - \frac{1}{C} \one \right)  \right\|_2 \\
= & \sqrt{\frac{C}{C-1}} \sqrt{2}.
\end{align*}
\end{proof}

\begin{thm}[\textbf{$\Delta$-optimality of Simplex ETF}]\label{thm:delta_optimality}
Let $\M^\star =  \sqrt{\frac{C}{C-1}} \left(\I - \frac{1}{C} \one \one^\top \right)$. Then,
\[
       \Delta_C^\star = \max_{ \| \M \|_{2,\infty} \leq 1 } \Delta(\M)  = \Delta(\Mstar)  = \sqrt{\frac{2C}{C-1}} .
\]
Moreover, the only matrices that achieve equality are $\Mstar$ or else matrices equivalent to it by orthogonal transformations from the left, $\M = \U \Mstar$, $\U^\T \U= \I$.
\end{thm}

\begin{proof}
The argument follows four steps. First, for any matrix $\M$ with column lengths $\|\bmu_c \|_2 \leq 1$,
there is another matrix $\widetilde{\M} $ with all columns of unit
length, obeying $\Delta(\widetilde{\M}) \geq \Delta(\M)$;
see Lemma \ref{lem:rescaleUnitNorm}. Hence, for determining the global maximizer,
\[
      \max_{\text{diag}(\M^\T\M) \leq \I } \Delta(\M)  =  \max_{ \text{diag}(\M^\T\M)= \I } \Delta(\M) .
\]
Thus, without loss of generality, we focus on the matrices $\M$ with column lengths $\|\bmu_c \|_2 = 1$, $c=1,\dots,C$.

Second, for any matrix $\M$ with column lengths $\|\bmu_c \|_2 = 1$,
the off-diagonal entries of $\M^\T\M$ are at least $\frac{-1}{C-1}$.
See Lemma \ref{lem:lowerBoundOffDiagonal}.

Third, for two vectors $\p$, $\q$ both of norm 1, $\|\p\|_2 = \|\q\|_2 = 1$,
the distance $\|\p - \q\|_2^2 = 2 - 2 \langle \p, \q \rangle$. 
Hence, if $ \langle \p, \q \rangle \geq \frac{-1}{C-1}$,
then 
\[
\|\p - \q\|_2^2 \leq 2 + \frac{2}{C-1} = \frac{2 C}{C-1}.
\]

Combining steps 1-3, 
\[
\max_{ \| \M \|_{2,\infty} \leq 1 } \Delta(\M)  \leq  \sqrt{\frac{2C}{C-1}}.
\]
However, from the  Lemma \ref {lem:EvalDeltaSimplex}, we already know
$\Delta(\M^\star) =  \sqrt{\frac{2C}{C-1}}$,
so we conclude that:
\[
  \Delta_C^\star = \sqrt{\frac{2C}{C-1}}.
\]

In step 4, we show that every $C$ by $C$ matrix $\M$
attaining equality
must be left-equivalent to $\M^\star$
by orthogonal rotation.
This is handled in Lemma \ref{lem:EquivToSimplexETF}.
\end{proof}

\begin{lem} \label{lem:minimalEnclosingSphere}
View the columns of $\M$ as $C$ points in $\R^C$ and suppose they
are affinely independent. There is a \textbf{unique} \textit{minimal enclosing
sphere} (MES), i.e. a sphere with minimal radius containing 
every point $\{\bmu_c\}_{c=1}^C$. Moreover, the MES has all $C$ points (columns of $\M$) on its surface.
\end{lem}
\begin{proof}
See \cite{fischer2003fast} and citations therein.
\end{proof}

\begin{lem} \label{lem:rescaleUnitNorm}
View the columns of $\M$ as $C$ points in $\R^C$ and suppose they
are affinely independent. Suppose that $\| \bmu_c \|_2 \leq 1$
for $c=1,\dots,C$ with strict inequality for some $c$.
For each such matrix $\M$, there is a corresponding 
$\widetilde{\M}$ whose columns obey exact normalization
$\| \tilde{\bmu}_c \|_2 = 1$, $c=1,\dots,C$.
Every intercolumn distance between a pair of 
columns of $\widetilde{\M}$ is 
strictly larger than between the corresponding pair of
columns of $\M$, i.e.
\[
   \| \tilde{\bmu}_c - \tilde{\bmu}_{c'} \|_2  >  \| {\bmu}_c - {\bmu}_{c'} \|_2 , \qquad c \neq c'.
\]
\end{lem}

\begin{proof}
By hypothesis, the standard unit ``solid''  sphere $\B_1$ contains
the points  $\{\bmu_c\}_{c=1}^C$. However, because at least one of the
points is interior to $\B_1$, the standard unit sphere $\cS^{C-1}$ is
{\it not} the MES of those points by Lemma \ref{lem:minimalEnclosingSphere}. The MES therefore has a 
radius $0 < r < 1$. The MES has a center $\p_0$, say, and we have
\[
     \| \bmu_c - \p_0 \|_2 = r, \qquad c=1,\dots,C.
\]
Define $\tilde{\bmu}_c = \frac{1}{r} (\bmu_c - \p_0)$. Then, $\| \tilde{\bmu}_c \| = 1$, while 
\begin{eqnarray*}
   \| \tilde{\bmu}_c - \tilde{\bmu}_{c'} \|_2 = \frac{1}{r}  \| {\bmu}_c - {\bmu}_{c'} \|_2 >  \| {\bmu}_c - {\bmu}_{c'} \|_2 \qquad \forall c \neq c'.
\end{eqnarray*}
\end{proof}

\begin{lem} \label{lem:lowerBoundOffDiagonal}
If  $\{ \bmu_c \}_{c=1}^C$ are $C$ points on the sphere  in $\R^C$ with $
 \max_{c' \neq c}   \langle \bmu_c ,  \bmu_{c'} \rangle =   \rho$,
then $\rho \geq \frac{-1}{C-1}$.
\end{lem}
\begin{proof}
The Gram matrix $\G =  (\langle \bmu_c ,  \bmu_{c'} \rangle) = \M^\T\M$ has diagonal entries $1$,
and off-diagonal entries $ \leq \rho$. Thus,
\[
    \one^\T \G \one \leq C + C(C-1) \rho .
\]
But, $\G$ is nonnegative semidefinite. Hence,
\[
    1 + (C-1) \rho \geq 0 \implies \rho \geq \frac{-1}{C-1}.
\]
\end{proof} 

\begin{lem} \label{lem:upperBoundOffDiagonal}
Suppose  $( \bmu_c )$ are $C$ points on the sphere  in $\R^C$ with
\[
 \max_{c' \neq c}   \langle \bmu_c ,  \bmu_{c'} \rangle \leq  \frac{-1}{C-1}.
\]
Then,
 \[
      \langle \bmu_c ,  \bmu_{c'} \rangle =  \frac{-1}{C-1}, \qquad c \neq c'.
\]
\end{lem}
\begin{proof}
The Gram matrix $\G =  (\langle \bmu_c ,  \bmu_{c'} \rangle) = \M^\T\M$ has diagonal entries $1$
and off-diagonal entries $ \leq \frac{-1}{C-1}$. By assumption,
\[
    \one^\top \G \one \leq C + C(C-1)  \frac{-1}{C-1} ,
\]
i.e. $\one^\top \G \one \leq 0$.
If for some specific pair $(c,c')$, it held that
$\langle \bmu_c ,  \bmu_{c'} \rangle <  \frac{-1}{C-1}$,
the inequality  would be strict: $\one^\T \G \one < 0$. 
As $\G$ is nonnegative semidefinite, $\one^\T \G \one \geq 0$,
and the inequality can never be strict, hence
\[
\langle \bmu_c ,  \bmu_{c'} \rangle =  \frac{-1}{C-1}, \qquad c \neq c'.
\]
\end{proof} 

\begin{lem} \label{lem:EquivToSimplexETF}
Suppose that $\M$ is a matrix
having all columns vectors of length $1$,
and every pair of interpoint distances 
\[
    \| \bmu_c - \bmu_{c'} \|_2 \geq \sqrt{\frac{2C}{C-1}}.
\]
Then, $\M = \U \Mstar$, where $\U^\T \U= \I$.
\end{lem}

\begin{proof}
The assumption
\[
\| \bmu_c - \bmu_{c'} \|_2 \geq  \sqrt{\frac{2C}{C-1}} , \qquad c \neq c' ,
\]
implies that
\[
 \langle \bmu_c , \bmu_{c'} \rangle \leq  \frac{-1}{C-1} , \qquad c \neq c' .
\]
However, Lemmas \ref{lem:lowerBoundOffDiagonal}-\ref{lem:upperBoundOffDiagonal} then imply
\[
 \langle \bmu_c , \bmu_{c'} \rangle =  \frac{-1}{C-1} , \qquad c \neq c' .
\]
Equivalently, since $\|\bmu_c\|_2=1$, $c=1,\dots,C$ 
by hypothesis,
\[
   \M^\T\M = (\Mstar)^\T\Mstar.
\]

\newcommand{\bj}{{\bf j}}
It remains to show that $\M = \U \Mstar$, where $\U$ is orthogonal. To this end, observe that the matrix $\G \equiv (\Mstar)^\T\Mstar$ is 
symmetric nonnegative definite with one eigenvalue $0$,
and all the other eigenvalues equal to $ \frac{C}{C-1}$.
The normalized eigenvector associated to
eigenvalue 0 may be 
taken as $ \bj= \one /\sqrt{C}$.
Spectral Decomposition gives 
$\G \equiv (\Mstar)^\T\Mstar  = \bV_{\Mstar} \bLambda_{\Mstar}  \bV^\T_{\Mstar}$.

The singular value decomposition of $\M = \U_{\M} \D \bV^\T_{\M}$ can be taken
to have $\bV_{\M} = \bV_{\Mstar}$, and with $\U_{\M}$ defined as follows:
First, set $ \U_{\M}^0 = \M  \bV_{\Mstar}   \sqrt{ \bLambda_{\Mstar} ^{\dagger}}$,
$ \U^0_{\Mstar}  = {\Mstar}   \bV_{\Mstar}   \sqrt{ \bLambda_{\Mstar} ^{\dagger}} $. 
One can check that these are each partial isometries,
omitting a one-dimensional range. 
Then, via a rank-one modification, we can generate the orthogonal matrices $\U_{\M}$ and $\U_{\Mstar}$.

We next verify that $\M = \U_{\M} \D \bV^\T_{\Mstar}$ is a valid SVD of $\M$, where $\D = \text{diag}(\bLambda_{\Mstar}^{1/2})$,
and that $\Mstar = \U_{\Mstar} \D \bV_{\Mstar}$ is also a valid SVD of $\Mstar$.
Set $\U = \U_{\M} \U^\T_{\Mstar} $; it is orthogonal. Then, $\M = \U {\Mstar} $, where ${\Mstar} $ is the matrix of the standard Simplex ETF. 
Therefore, $\M$ is also the matrix of a Simplex ETF, only not the standard one.

\end{proof}

\section{$\beta$-Optimality of the Simplex Tight Frame}\label{sec:beta_optimality}

\subsection{LD exponent for the Simplex Tight Frame}
\newcommand{\Sccp}{{\cal S}_{c,c'}}
\newcommand{\m}{\boldsymbol{m}}
\newcommand{\N}{\boldsymbol{N}}
\newcommand{\g}{\boldsymbol{g}}
\renewcommand{\u}{\boldsymbol{u}}
\newcommand{\Wstar}{{\W^\star}}
\begin{lem} \label{eq:MstarSymmetry}
Solve the instance $\beta(\Mstar,\Mstar,\zr)$ of the optimization problem $\beta(\M,\W,\b)$ defined
in \eqref{eq:PMAxDef}. The solution 
$\z^\star = (\z^\star_{c,c'})$ obeys:
\begin{equation}
       \| \z^\star_{c,c'} \|_2 = \frac{1}{2} \| \bmu^\star_c - \bmu^\star_{c'} \|_2 , \qquad c' \neq c. \label{eq:MidPoint}
\end{equation}
\end{lem}
\begin{proof}
For a given linear classifier rule $\hat{\gamma}(\h) = \hat{\gamma}(\h; \W,\b)$,
define the {\it decision regions}
$\Gamma_c \equiv \Gamma_c(\W,\b) \equiv  \{ \h :  \hat{\gamma}(\h) = c \}$,
$c = 1,\dots,C$. Note that these regions are invariant under simultaneous rescaling
of $(\W,\b) \mapsto ( a \W , a \b)$ for $a > 0$:
\[
     \Gamma_c(a \W, a \b) = \Gamma_c( \W,  \b), \qquad a > 0, \quad c' \neq c.
\]

Put $\hat{\gamma}^\star \equiv \hat{\gamma}(\bh;\Mstar,\Mstar,\zr)$,
and define the decision regions $\Gamma_c^\star = \{\h : \hat{\gamma}^\star(\h) = c \}$, $c =1,\dots, C$.
Since the decision regions $\Gamma_c^\star$ do not change 
under a global rescaling of the $\W$-matrix,
we propose that, instead of using the announced matrix $\W=\M^\star$, we instead use
the rescaled matrix $\W^\star = \sqrt{\frac{C-1}{C}} \M^\star$. Namely, since
$\Gamma_c^\star = \Gamma_c(\W^\star,\zr)$, we compute the latter one.
Note that $\W^\star$ has all singular values $1$ or $0$, so it is a partial isometry,
which will have calculational advantages.

Let $\e_{c,c'}^\star$ denote the Euclidean closest member of $\Gamma_{c'}$ to 
$\bmu_c$.  An alternate, but equivalent, way of describing the optimization
problem $\beta(\Mstar,\W^\star,\zr)$ is to say that
\[
\bmu_c^\star + \z_{c,c'}^\star = \e^\star_{c,c'} .
\]
In short, $\z_{c,c}^\star$ is precisely the least Euclidean norm displacement that
can translate from $\bmu_c^\star$ to a member of $\Gamma_{c'}^\star$, and the closest point 
 in $\Gamma^\star_{c'}$ arrived at in this way is precisely $\e^\star_{c,c'}$.
Hence,
\[
\| \z^\star_{c,c'} \|_2 = \| \e^\star_{c,c'} - \bmu_{c}\|_2 = \mbox{dist}( \bmu_c, \Gamma_{c'}^\star ) .
\]
Define,
\[
\m^\star_{c,c'} \equiv (\bmu_{c} + \bmu_{c'})/2, \qquad c \neq c',
\]
i.e. the halfway point between $\bmu^\star_c$ and $\bmu^\star_{c'}$.
The statement to be proved, \eqref{eq:MidPoint}, is therefore equivalent to
\begin{equation}
       \e^\star_{c,c'} = \m^\star_{c,c'} , \qquad c \neq c'. \label{eq:ClosestIsMidpoint}
\end{equation}
We will verify that the candidate $\m^\star_{c,c'}$
is indeed the Euclidean closest point to $\bmu^\star_c$ 
within $\Gamma_{c'}^\star$.
Such a candidate point is actually the halfway point
along the line segment $\Sccp$ joining $\bmu_c$ to   $\bmu_{c'}$.
The candidate point is, therefore, 
identical to the closest point in $\Gamma_{c'}^\star$ 
exactly when:
\begin{description}
\item[{[a]}] the candidate point is  on the decision boundary;
\item[{[b]}] the decision boundary 
is orthogonal to said line segment.
\end{description}

The decision boundary is, more explicitly,
\[
 \partial \Gamma_{c'}^\star = \partial \Gamma_{c'}(\W^\star,\zr) =  \{ \h : (\W^\star \h)(c) = (\W^\star \h)(c') \}.
\]
We first show [a]: that $\m_{c,c'}^\star \in \partial \Gamma_{c'}(\W^\star,\zr)$.
Clearly,
\[
   \W^\star\m^\star_{c,c'} = \frac{1}{2} (\W^\star\bmu^\star_c + \W^\star\bmu^\star_{c'}). 
\]
Hence, 
\begin{gather*}
(\W^\star\m^\star_{c,c'})(c) =  \frac{(\W^\star\bmu^\star_c)(c) + (\W^\star\bmu^\star_{c'})(c)}{2},\text{ and }\\
(\W^\star\m^\star_{c,c'})(c') =  \frac{(\W^\star\bmu^\star_c)(c') + (\W^\star\bmu^\star_{c'})(c')}{2}.
\end{gather*}
Now, because $\W^\star = \sqrt{\frac{C-1}{C}} \M^\star$, which is symmetric and a partial isometry,
\begin{equation}
   \W^\star \M^\star = \M^\star. \label{eq:Invariant}
\end{equation}
Hence, $\W^\star\bmu^\star_c= \bmu_c^\star$ and $\W^\star\bmu^\star_{c'}= \bmu_{c'}^\star$.
Our explicit formula for $\Mstar$ shows that all {\bf on-diagonal}
terms are equal to each other and all {\bf off-diagonal} terms 
are equal to each other. Hence, these off-diagonal terms obey
\[
\bmu^\star_c(c') = \bmu^\star_{c'}(c) = \frac{-1}{\sqrt{C(C-1)}};
\]
and the on-diagonal ones
\[
\bmu^\star_c(c) = \bmu^\star_{c'}(c') = \sqrt{\frac{C-1}{C}}.
\]
It follows that
\[
 \m^\star_{c,c'}(c) = \m^\star_{c,c'}(c').
\]
Combining with \eqref{eq:Invariant} we obtain:
\[
[\W^\star\m^\star_{c,c'}](c) = \m^\star_{c,c'}(c) = \m^\star_{c,c'}(c') = [\W^\star\m^\star_{c,c'}](c');
\]
i.e. $\m^\star_{c,c'} \in \partial \Gamma_{c'}(\W^\star,\zr)$; the candidate point 
is in the decision boundary, namely [a].

We now consider [b]: orthogonality.
Define the linear space $\N = \{ \h : (\W^\star \h)(c) - (\W^\star \h)(c') = 0\} $
and the linear space $\G =  \mbox{lin}(\Sccp - \m^\star_{c,c'}) $,
where $\mbox{lin}()$ denotes linear span.
Our orthogonality assertion is equivalent to
\[
      \langle \g , \h \rangle = 0 , \qquad \h \in \N, \g \in \G.
\]
Define $\u = (\bmu^\star_c - \bmu^\star_{c'})/2$; in fact $\G = \mbox{lin}(\{\u\})$.
So, we must show
\begin{equation}
      \langle \u , \h \rangle = 0 , \qquad \forall \h \in \N. \label{eq:SimplifiedOrthogonality}
\end{equation}
Now each $\h \in \N$ can be decomposed as $\h=\h_0+\h_1$ where $\h_0 \in \mbox{ker}(\Wstar)$
while $\h_1 \in \mbox{range}(\Wstar)$. Explicit formulas for $\Wstar$ show that
\[
\mbox{ker}(\Wstar) = \mbox{lin}(\{ \one \}).
\]
Hence, $\h_0(c) = \h_0(c')$. On the other hand, $\Wstar \h_1 = \h_1$. Combining these two,
if $(\W^\star\h)(c) = (\W^\star\h)(c')$, then  $\h_1(c) = \h_1(c')$; and since always
$\h_0(c) = \h_0(c')$, we obtain $\h(c) = \h(c')$. Rewriting [\ref{eq:SimplifiedOrthogonality}]
as
\[ 
      \langle \u, \h \rangle = 0 , \qquad \forall \h: (\W^\star\h)(c) = (\W^\star\h)(c'),
\]
we see this is equivalent to 
\[ 
      \langle \u, \h \rangle = 0 , \qquad \forall \h: \h (c) = \h(c').
\]
Let $\bdelta_{c}$ denote the Kronecker sequence, $\one_{\{c'=c\}}(c')$. By the explicit form definitions of $\Mstar$ and $\u$,
\[
\u\propto \left(\bdelta_c - \frac{1}{C} \one\right) - \left(\bdelta_{c’} - \frac{1}{C} \one\right)=\bdelta_c - \bdelta_{c’}.
\]
Thus,
\[
 \langle (\bdelta_c - \bdelta_{c'}) ,  \h \rangle = 0 , \qquad \forall \h: \h (c) = \h(c'),
\]
i.e.
\[
 \h(c) - \h({c'}) = 0, \qquad \forall \h: \h (c) = \h(c'),
\]
which of course is true. This establishes orthogonality, [b], and 
completes the demonstration of \eqref{eq:ClosestIsMidpoint},
and hence of the main claim \eqref{eq:MidPoint}.
\end{proof}

\begin{crl} \label{cor:BetaSimplexEquality}
We have
\[
\beta(\Mstar,\Mstar,\zr) = \frac{1}{8} \Delta(\Mstar)^2.
\]
\end{crl}

\begin{proof}
By our earlier definitions, if $\z^\star = (\z_{c,c'}^\star)$ denotes a solution to
$(P_{\Mstar,\M^\star,\zr})$, then
\[
\beta(\Mstar,\Mstar,\zr) = \frac{1}{2} \min_{c' \neq c} \| \z_{c,c'}^\star \|_2^2 .
\]
By the previous lemma,
\[
\| \z_{c,c'}^\star \|_2 = \frac{1}{2} \| \bmu^\star_{c} - \bmu^\star_{c'} \|_2 .
\]
Combining these two identities,
\[
\beta(\Mstar,\Mstar,\zr) = \frac{1}{8}\min_{c' \neq c}  \| \bmu^\star_{c} - \bmu^\star_{c'} \|_2^2 = \frac{1}{8} \Delta(\Mstar)^2.
\]
\end{proof}

\subsection{Proof of Theorem 5}\label{subsec:proof_thm5}

\begin{proof}
In view of the inequality [\ref{eq:BetaOptBound}] and Theorem \ref{thm:delta_optimality}, we know that 
\[
\beta_C^{\star} \leq \frac{1}{8} (\Delta_C^\star)^2 = \frac{1}{4} \cdot \frac{C}{C-1} .
\]
From Corollary \ref{cor:BetaSimplexEquality} and Theorem \ref{thm:delta_optimality}, 
we know that equality holds for the standard Simplex ETF:
\begin{equation} \label{eq:SimplexBetaEquality}
\beta(\Mstar,\Mstar,0) =  \frac{1}{4} \cdot \frac{C}{C-1}.
\end{equation}
Hence, $\beta^{\star}_C = \beta(\Mstar,\Mstar,0)$; the Simplex ETF is $\beta$-optimal.
It follows by orthogonal invariance of the decision problem,
that for a Simplex ETF in any isometric pose, equality also holds:
\[
\beta(\U\Mstar,\Mstar \U^\T,0) =  \frac{1}{4} \cdot \frac{C}{C-1}; \qquad \forall \U, \quad \U^\T\U=I.
\]
So, Simplex ETF's are all optimal.
Finally, since such $\M$s are {\bf the only} solutions to $\Delta^\star(\M) = \Delta_C^\star$ obeying $\| \M \|_{2,\infty} \leq 1$, suppose we have some
{\bf  other} candidate $\breve{\M}$, obeying $\| \breve{\M} \|_{2,\infty} \leq 1$ but
{\it not} obeying $\breve{\M}=\U\Mstar$ for some orthogonal matrix $\U$. 
Then, Theorem \ref{thm:delta_optimality} implies $\Delta(\breve{\M}) < \Delta_C^\star$,  and so,
\[
\frac{1}{8} \Delta(\breve{\M})^2 <  \frac{1}{8} (\Delta_C^\star)^2 = \frac{1}{4} \cdot \frac{C}{C-1}.
\]
Applying inequality [\ref{eq:betaDeltaIndivBnd}] to such a candidate $\breve{\M}$, we have
\[
\max_{\W \in L(\V,\V), \b \in \V}  \beta(\breve{\M},\W,\b) \leq   \frac{1}{8} \Delta(\breve{\M})^2 < \frac{1}{4} \cdot \frac{C}{C-1} = \beta^\star_C.
\]
In short, any such candidate is suboptimal. We have thus described all choices of $\M$ achieving $\beta_C^{\star}$;
just as claimed.
\end{proof}

\FloatBarrier






\end{document}